\documentclass{article}

 \usepackage[main, final]{neurips_2026}

\usepackage[utf8]{inputenc} 
\usepackage[T1]{fontenc}    
\usepackage{hyperref}       
\usepackage{url}            
\usepackage{booktabs}       
\usepackage{amsfonts}       
\usepackage{nicefrac}       
\usepackage{microtype}      
\usepackage{xcolor}         

\usepackage{amsmath}
\usepackage{amssymb}
\usepackage{amsthm}
\usepackage{amsfonts}
\usepackage{stmaryrd}
\usepackage{array}
\usepackage{mathrsfs}
\usepackage{mathtools}
\usepackage{multirow} 
\usepackage{float}
\usepackage{subcaption}
\newcolumntype{L}[1]{>{\raggedright\arraybackslash}p{#1}}
\newcolumntype{C}[1]{>{\centering\arraybackslash}p{#1}}
\newcolumntype{R}[1]{>{\raggedleft\arraybackslash}p{#1}}
\def\R{{\mathbb R}}
\def\E{{\mathrm E}}
\def\cF{{\mathcal F}}
\def\P{{\mathrm P}}
\newcommand{\Var}{\mathrm{Var}}
\newcommand{\Cov}{\mathrm{Cov}}

\newtheorem{theorem}{Theorem}
\newtheorem{lemma}{Lemma}
\newtheorem{proposition}{Proposition}[section]
\newtheorem{remark}{Remark}
\newtheorem{assumption}{Assumption}

\title{Nearest-Neighbor Radii under Dependent Sampling}

\author{%
  Yuanyuan Gao$^{\ast}$ \\
  Department of Statistics\\
  University of California, Berkeley\\
  Berkeley, CA 94720 \\
  \And
  Yilong Hou$^{\ast}$ \\
  Department of Biostatistics\\
  University of California, Berkeley\\
  Berkeley, CA 94720 \\
  \AND
  Zhexiao Lin$^{\ast,\dagger}$ \\
  Department of Statistics\\
  University of California, Berkeley\\
  Berkeley, CA 94720 \\
}

\begin{document}

\maketitle

\footnotetext{$^\ast$Authors are listed in alphabetical order.}
\footnotetext{$^\dagger$Corresponding author: \texttt{zhexiaolin@berkeley.edu}.}

\begin{abstract}
Nearest-neighbor methods are fundamental to classical and modern machine learning, yet their geometric properties are typically analyzed under independent sampling. In this paper, we study the nearest-neighbor radii under dependent sampling. We consider strong mixing dependent observations and ask whether dependence changes the scale of nearest-neighbor neighborhoods. We establish distribution-free almost sure convergence under polynomial mixing and sharp non-asymptotic moment bounds under geometric mixing. The moment bounds depend on the local intrinsic dimension rather than the ambient dimension, making the results applicable to high-dimensional data supported on lower-dimensional manifolds. Synthetic experiments and real-world time-series benchmarks illustrate the theory, showing that nearest-neighbor geometry remains informative under dependent sampling.
\end{abstract}

\section{Introduction}\label{sec:intro}

Nearest-neighbor (NN) methods remain one of the most fundamental ideas in machine learning. 
They underlie classical procedures such as $k$-nearest-neighbor classification \citep{cover1967nearest} and regression \citep{stone1977consistent}, 
nonparametric density estimation and local averaging \citep{loftsgaarden1965nonparametric}, and they also appear implicitly in a much broader range of modern techniques built on local neighborhoods, including graph-based semi-supervised learning \citep{zhu2003semi}, manifold methods \citep{belkin2003laplacian}, information retrieval \citep{wang2021comprehensive}, and language modeling \citep{khandelwal2020nearest,xu2023k, xu2023nearest}. In all of these methods, a basic geometric quantity at the center of the analysis is the distance from a query point $x$ to its $k$th nearest neighbor in the sample. This random radius determines the scale of the local neighborhood around $x$, and therefore controls how much of the ambient space the method actually ``sees.''

From a statistical perspective, nearest-neighbor radii determine the central bias--variance tradeoff of local methods \citep{stone1977consistent, biau2015lectures}. If the radius is too large, a local estimator averages over points that are geometrically far from the query and thus introduces bias. If the radius is too small, the estimator uses too few effective observations and becomes unstable. In density estimation, the $k$-NN radius is the random bandwidth that determines the volume of the adaptive neighborhood. In regression and classification, it controls the locality of smoothing. In manifold and graph-based methods, it determines neighborhood connectivity and the scale at which local geometry is resolved. Understanding the scaling of the $k$-NN radius is therefore a prerequisite for understanding the performance of a wide class of machine learning algorithms.

In the classical i.i.d.\ setting, the behavior of $k$-NN radius is well understood: under standard local mass conditions, the canonical scale is $(k/n)^{1/d}$, or more generally, $(k/n)^{1/s}$ when the local geometry is determined by an intrinsic dimension $s$ \citep{loftsgaarden1965nonparametric,biau2015lectures,kpotufe2011k}. However, many learning problems do not arise from independent samples. Data are often collected sequentially, spatially, or through interaction with a dynamical environment \citep{mohri2010stability,tu2024learning}. In such settings, observations can be strongly dependent even when their marginal distribution is stable. This raises a basic question:
\begin{center}
\vspace{-2mm}
\emph{How does dependence affect the geometry of nearest-neighbor neighborhoods?}
\vspace{-2mm}
\end{center}
This question is both mathematically natural and practically important. A large body of machine learning theory relies on i.i.d.\ nearest-neighbor geometry either explicitly or implicitly. Yet once dependence is present, standard concentration arguments no longer apply directly, and even basic geometric quantities such as the $k$-NN radius become more delicate. If dependence were to fundamentally change the scaling of nearest-neighbor radii, then the geometry underlying many local learning procedures would also change. On the other hand, if the canonical radius scale remains unchanged under weak dependence, this would help justify the use of nearest-neighbor and local methods beyond the i.i.d. regime.

In this paper, we study the nearest-neighbor radii under dependent sampling. We consider a triangular array of $\R^d$-valued observations with a common marginal distribution and strong mixing dependence along each row. This framework is general enough to cover the usual stationary-sequence setting as a special case, while also allowing the dependence structure to vary with $n$. Our object of study is the radius $R_{n,k}(x) := \| \mathcal X_{(k)}^n(x)-x\|$,
that is, the distance from a fixed query point $x$ to its $k$th nearest neighbor among the $n$ observations in the row.

Our first result is almost sure convergence: under polynomial $\alpha$-mixing, we show that the nearest-neighbor radius converges almost surely to the distance from $x$ to the support of the marginal law. This gives a dependent-sampling analogue of the classical consistency result and shows that weak dependence does not change the limiting geometry of nearest-neighbor radii. We then develop a non-asymptotic theory under geometric mixing. Under a local lower-mass condition of order $s$, we establish a tail bound for $R_{n,k}(x)$ and derive the sharp moment upper bound $\E[R_{n,k}(x)^p] \lesssim (k/n)^{p/s}$ in the standard regime $k \gtrsim \log n$. Under a matching local upper-mass condition of order $s$, we also establish the matching lower bound and hence obtain the exact order $\E[R_{n,k}(x)^p] \asymp (k/n)^{p/s}$. Thus, geometric dependence changes the concentration mechanism, but not the fundamental exponent.

A key feature of our theory is that it is formulated in terms of a local exponent $s>0$, rather than the ambient dimension $d$. This point is especially relevant for modern machine learning, where high-dimensional data often concentrate near lower-dimensional latent structures \citep{fefferman2016testing,ansuini2019intrinsic}. Our results therefore support the common regime in which the effective geometry around the query point is lower-dimensional than the ambient space. In this sense, our results are not only about dependence, but also about how dependence interacts with local geometric complexity.

\vspace{-3mm}
\paragraph*{Our contributions.}
The main contributions are as follows.
\vspace{-2mm}
\begin{enumerate}
    \item We establish \emph{distribution-free almost sure convergence} for the $k$-NN radius under polynomial $\alpha$-mixing. This extends the classical i.i.d. consistency result to dependent triangular arrays.

    \item Under geometric $\alpha$-mixing and a local lower-mass condition of order $s$, we establish a \emph{non-asymptotic tail bound} for the $k$-NN radius and derive the \emph{sharp moment upper bound} $\E[R_{n,k}(x)^p] \lesssim (k/n)^{p/s}$ for all $p>0$ in the standard dependent regime $k \gtrsim \log n$.

    \item Under a local upper-mass condition, we establish a \emph{matching lower bound} $\E[R_{n,k}(x)^p] \gtrsim (k/n)^{p/s}$, and hence obtain the \emph{exact-order characterization} $\E[R_{n,k}(x)^p] \asymp (k/n)^{p/s}$.

    \item Our theory is stated in terms of a \emph{local geometric exponent} $s$, which yields an intrinsic-dimension interpretation and is relevant for modern high-dimensional learning problems.
\end{enumerate}

\vspace{-2mm}
We believe the main message is simple and useful: under weak dependence, nearest-neighbor geometry remains essentially the same as in the independent case at the level of its canonical scale. This provides theoretical support for transferring geometric intuition and local-method design principles from the i.i.d.\ world to substantially broader dependent settings.

\vspace{-2mm}
\subsection{Related work}
\label{sec:related}
\vspace{-2mm}

\paragraph*{Applications of nearest-neighbor methods.} In recent years, nearest-neighbor methods have re-emerged in modern machine learning as a flexible mechanism for retrieval, adaptation, and locality in learned representation spaces. Examples include self-supervised representation learning \citep{avdiukhin2024embedding,long2025mnn}, tabular deep learning \citep{gorishniy2023tabr}, vector search and approximate nearest-neighbor infrastructure \citep{chen2021spann,yang2024cspg,jaasaari2024lorann,wang2021comprehensive}, and robust retrieval and vision \citep{wu2022retrievalguard,nizan2024k}. Beyond core machine learning, nearest-neighbor structure also plays an important role in causal inference through matching estimators \citep{abadie2006large,abadie2011bias,lin2023estimation,lin2025regression,cattaneo2025rosenbaum}, in rank-based and graph-based measures of dependence \citep{chatterjee2021new,azadkia2021simple,han2021extensions,lin2022limit,lin2023boosting,lin2024failure, shi2024azadkia, han2024azadkia}, in information-theoretic functional estimation via $k$-NN entropy and mutual-information estimators \citep{kozachenko1987sample,kraskov2004estimating}, and in multivariate nonparametric testing based on geometric graphs \citep{friedman1979multivariate,henze1999multivariate}. This breadth of applications makes the geometric behavior of nearest-neighbor neighborhoods a foundational issue.

\paragraph*{Theory of nearest-neighbor methods under dependence.} Despite the wide range of applications, previous literature on the theory under dependence has mostly analyzed estimators built from nearest-neighbor neighborhoods, rather than the geometry of the neighborhoods themselves. Examples include nearest-neighbor classification with dependent training sequences \citep{holst2001nearest} and rates for nearest-neighbor estimation under more general non-i.i.d.\ sampling schemes \citep{kulkarni1995rates}. More broadly, the literature on learning from dependent observations has developed both positive and negative results for nonparametric prediction under mixing, ergodic, and other non-i.i.d.\ assumptions \citep{yu1994rates, nobel1999limits,mohri2010stability}. By contrast, our focus is on the radius itself as a primary object. This shift in perspective is useful because the nearest-neighbor radius serves as the random local scale in many downstream methods, so controlling it directly provides a geometric ingredient for analyzing local learning methods. Related work gives oracle inequalities for regularized empirical risk minimization under stationary $\alpha$-mixing \citep{hang2014fast}, concentration for Banach-valued sums and kernel learning under $\tau$-mixing \citep{blanchard2019concentration}, applications of $U$-statistic concentration for uniformly ergodic Markov chains \citep{duchemin2022three}, and generalization bounds for deep networks with $\tau$-mixing samples \citep{liu2025generalization}. These works analyze sums, risks, or testing statistics; our focus is the inversion of dependent small-ball counts to characterize the nearest-neighbor radius.

\section{Setup}
\label{sec:setup}

We study the geometry of nearest neighbors under dependent sampling. Let $[X_{n,i}: 1 \le i \le n,\ n \ge 1]$ be a triangular array of $\R^d$-valued random variables defined on a common probability space $(\Omega,\cF,\P)$, where the ambient dimension $d \ge 1$ is fixed. For each row $n$, the observations $X_{n,1},\dots,X_{n,n}$ could be independent, but we assume that they share a common marginal law $\mu$ on $\R^d$. Thus, for every $n \ge 1$ and every $i \in \{1,\dots,n\}$, $X_{n,i} \sim \mu$.

This formulation distinguishes two roles in our analysis. The common marginal distribution $\mu$ determines the local geometry of the marginal distribution, while the dependence structure controls how much effective information is contained in a sample of size $n$. Allowing the dependence structure to vary with $n$ leads naturally to a triangular-array framework, which is flexible enough to cover the usual stationary-sequence setting as a special case and at the same time convenient for stating non-asymptotic results uniformly in $n$.

\paragraph*{Triangular-array strong mixing.} For integers $1 \le k \le m \le n$, let ${}^n\cF_k^m := \sigma(X_{n,i} : k \le i \le m)$ denote the $\sigma$-field generated by the segment $(X_{n,k},\dots,X_{n,m})$ in the $n$th row. Following the standard strong-mixing framework, define the mixing coefficients of the triangular array by
\vspace{-1mm}
\[
\alpha_m
:=
\sup_{n \ge m+1}\ \sup_{1 \le k \le n-m}\ 
\sup\Big\{
\big| \P(A \cap B) - \P(A)\P(B) \big|
:\ 
A \in {}^n\cF_1^k,\ 
B \in {}^n\cF_{k+m}^n
\Big\},
\qquad m \ge 1.
\]
The array is said to be \emph{strong mixing} (or \emph{$\alpha$-mixing}) if $\alpha_m \to 0$ as $m \to \infty$.

The strong-mixing framework is classical in probability and time-series analysis \citep{bradley2005basic}. This definition requires a uniform decay of dependence across all rows of the array. Intuitively, if two collections of variables are separated by at least $m$ indices within the same row, then their dependence becomes weak when $m$ is large. This is the dependence notion used throughout the paper.

The triangular-array formulation is more general than working with a single stationary process $[X_i]_{i\ge1}$. It allows the dependence structure to vary with $n$, while keeping the marginal distribution fixed. We state finite-sample results uniformly in $n$, which also covers the usual time-series setting as a special case by taking $X_{n,i} = X_i$ for $1 \le i \le n$. Therefore, readers who prefer the standard sequential viewpoint may regard our framework as a generalization.

\paragraph*{Nearest-neighbor radii.}
Fix a query point $x \in \R^d$. For each $n \ge 1$ and $k \in \{1,\dots,n\}$, let $\mathcal X_{(k)}^n(x) \in \{X_{n,1},\dots,X_{n,n}\}$ denote the $k$th nearest neighbor of $x$ among the sample points in the $n$th row, with ties broken arbitrarily. We introduce the $k$-NN radius:
\[
  R_{n,k}(x) := \|\mathcal X_{(k)}^n(x)-x\|,
\]
together with the local counting process: $N_n(x,r):= \sum_{i=1}^n \mathbf{1}\{\|X_{n,i}-x\| \le r\}$. These two quantities are equivalent through the relationship $R_{n,k}(x) > r \Longleftrightarrow N_n(x,r) < k$. This observation is the starting point for our entire analysis: upper bounds on the radius are equivalent to lower bounds for the number of sample points falling in $B(x,r)$. Under independence, $N_n(x,r)$ is binomial with parameters $n$ and $\mu(B(x,r))$, so its behavior can be analyzed using classical binomial tail bounds. Under dependence, however, this analysis is not applicable since the indicators $\mathbf{1}\{\| X_{n,i}-x\| \le r\}$ are dependent, and hence the local counting process is not binomial. Controlling $R_{n,k}(x)$ therefore requires concentration arguments for dependent sequences, which is one of the main challenges.

\paragraph*{Support and target limit.}
The asymptotic behavior of $R_{n,k}(x)$ depends on the position of the query point relative to the support of $\mu$. Define ${\rm supp}(\mu)
:=
\{x \in \R^d : \mu(B(x,r)) > 0 \text{ for all } r>0\}$, where $B(x,r)$ denotes the closed Euclidean ball centered at $x$ with radius $r$. We define the distance from $x$ to the support of $\mu$ by $r_x := \inf_{y \in {\rm supp}(\mu)} \| y-x \|.$ When $x \in {\rm supp}(\mu)$, one has $r_x=0$; otherwise $r_x$ is the deterministic limit that the $k$-NN radius converges to under mild assumptions.

Table~\ref{tab:notation} in Appendix~\ref{sec:proof} summarizes the notation used in the results and proofs.

\section{Main results}\label{sec:main}

We now present the main results for the $k$-NN radius under dependent sampling. As discussed in Section~\ref{sec:setup}, the behavior of $R_{n,k}(x)=\| \mathcal X_{(k)}^n(x)-x\|$ depends on both the dependence structure of the sample and the local geometry of the marginal law around the query point $x$. The almost sure limit is determined by the position of $x$ relative to the support of $\mu$, whereas the convergence rate depends on how the mass of $\mu$ accumulates near $x$. 


\subsection{Distribution-free almost sure convergence}

Our first theorem shows the almost sure convergence of the $k$-NN radius to $r_x$ under weak dependence.

\begin{theorem}\label{thm:as}
Assume that the triangular array $[X_{n,i}]_{1 \le i \le n,\ n \ge 1}$ is $\alpha$-mixing with coefficients $[\alpha_m]_{m \ge 1}$ satisfying $\alpha_m \lesssim m^{-\gamma}$ for some $\gamma>1$. If $k/n \to 0$, then $R_{n,k}(x) \stackrel{\sf a.s.}{\longrightarrow} r_x.$ In particular, if $x \in {\rm supp}(\mu)$ and $k/n \to 0$, then $R_{n,k}(x) \stackrel{\sf a.s.}{\longrightarrow} 0$.
\end{theorem}
\vspace{-2mm}
Theorem~\ref{thm:as} may be viewed as the dependent-sampling analogue of the classical i.i.d. consistency result \citep[Lemma 2.2]{biau2015lectures}. It shows that weak dependence, quantified here through a polynomial mixing condition, does not change the limiting geometric behavior of nearest-neighbor radii. Importantly, the result remains distribution-free: no density, smoothness, or local mass condition is required. The only assumption is that dependence decays polynomially fast along each row of the triangular array.

\subsection{From almost sure convergence to moment bounds}\label{sec:assumptions}

We next turn to non-asymptotic moment results. For these stronger results, one must quantify how much probability mass lies near the query point $x$.

Rather than working with the ambient dimension $d$, we state the results using a local exponent $s>0$. This formulation is more flexible and captures the possibility that the effective local dimension near $x$ is smaller than the ambient dimension. The standard Euclidean case is recovered by setting $s=d$.

For the non-asymptotic results, we impose the stronger assumption that the mixing coefficients decay exponentially fast.
\begin{assumption}[Geometric $\alpha$-mixing]\label{ass:geom}
There exists $\gamma>0$ such that $\alpha_m \le e^{-\gamma m}$ for all $m \ge 1$.
\end{assumption}
  
Assumption~\ref{ass:geom} strengthens the assumption in Theorem~\ref{thm:as} from polynomial to exponential decorrelation. This stronger assumption enables non-asymptotic concentration via blocking arguments. The parameter $\gamma$ quantifies the strength of dependence: larger $\gamma$ corresponds to faster decorrelation, and in the limit of vanishing dependence, one recovers behavior close to the i.i.d. case.

Geometric mixing is a standard assumption in time-series theory and is satisfied by many classical models under mild regularity conditions \citep{doukhan2012mixing}. In particular, stationary ARMA processes are known to be strongly mixing with exponential rate under standard stability and innovation-regularity conditions \citep{mokkadem1988mixing}, which makes geometric mixing a natural benchmark for studying dependent nearest-neighbor geometry. 

\begin{remark}[Other mixing conditions]\label{rem:other-mixing}
We use $\alpha$-mixing because it directly controls measurable ball events and is weaker than the $\beta$- and $\varphi$-mixing conditions. The radius bounds extend immediately to $\beta$-mixing and $\varphi$-mixing, at the same polynomial rate as in Theorem~\ref{thm:as} and the same geometric rate as in Assumption~\ref{ass:geom}. They do not extend to $\mathcal C$-mixing by the arguments of this paper.
For any two $\sigma$-fields, $2\alpha(\mathcal A,\mathcal B)\le\beta(\mathcal A,\mathcal B)\le\varphi(\mathcal A,\mathcal B)$ \citep[eq.~(1.11)]{bradley2005basic}, and the same comparison holds for the coefficients of each row, so $\alpha_m\le\beta_m/2\le\varphi_m/2$. A triangular array whose $\beta$- or $\varphi$-coefficients satisfy either decay condition is therefore $\alpha$-mixing at that same rate, and every result applies with no further argument. The converse implications fail in general \citep[Section~2.1]{bradley2005basic}.

The obstruction for $\mathcal C$-mixing is the form of the statistics in the proofs. $\mathcal C$-mixing controls correlations against a prescribed class of regular functions, such as Lipschitz or H{\"o}lder functions \citep{maume2006exponential,hang2017bernstein}. Ball indicators are discontinuous on the sphere and do not have a finite Lipschitz or H{\"o}lder $\mathcal C$-norm, so those inequalities do not apply directly to $N_n(x,r)$; a separate approximation argument would be needed.
\end{remark}

\begin{assumption}[Local lower mass of order $s$ at $x$]\label{ass:lower}
There exist constants $c_->0$, $r_0>0$, and $s>0$ such that $\mu(B(x,r)) \ge c_- r^s$ for all $0<r\le r_0$.
\end{assumption}
  
Assumption~\ref{ass:lower} guarantees that the ball $B(x,r)$ contains enough probability mass for its radius-$r$ neighborhood to capture roughly $nr^s$ sample points. This lower bound is the key ingredient for the upper bound of $R_{n,k}(x)$.

\begin{assumption}[Local upper mass of order $s$ at $x$]\label{ass:upper}
There exist constants $c_+>0$, $r_0>0$, and $s>0$ such that $\mu(B(x,r)) \le c_+ r^s$ for all $0<r\le r_0$.
\end{assumption}

Assumption~\ref{ass:upper} matches Assumption~\ref{ass:lower} and prevents the distribution from placing too much mass near $x$, and is needed to establish a matching lower bound on the radius.
  
\begin{assumption}[Compact support]\label{ass:compact}
The support of $\mu$ is contained in a compact set of diameter $D$.
\end{assumption}

Assumption~\ref{ass:compact} is for technical simplicity rather than a geometric limitation. The local behavior of the $k$-NN radius is determined by Assumptions~\ref{ass:lower} and~\ref{ass:upper}, which quantify how the mass of $\mu$ accumulates near the query point $x$. By contrast, compact support is used to control the far tail of $R_{n,k}(x)$ when converting non-asymptotic tail bounds into moment bounds. Proposition~\ref{prop:moment-noncompact} replaces compact support by $\E\|X_{n,1}-x\|^q<\infty$ for some $q>p$, covering the Gaussian design of Experiment~2.

\paragraph*{Interpretation of the local exponent $s$.}
The exponent $s$ is the \emph{local geometric dimension} around the query point $x$. Assumptions~\ref{ass:lower} and~\ref{ass:upper} say that, for sufficiently small radii, $\mu(B(x,r)) \asymp r^s$. Thus, the amount of probability mass captured by a small neighborhood of $x$ grows like the volume of an $s$-dimensional ball. In the standard Euclidean case with a positive density in $\R^d$, one has $s=d$. More generally, $s$ may be strictly smaller than $d$, e.g., when the distribution is supported on a lower-dimensional manifold, subspace, or other latent geometric structure. Fixed full-dimensional noise can instead give $s=d$ at sufficiently small radii. Our theory is therefore local and intrinsic: the scaling of the $k$-NN radius is determined by the intrinsic dimension of the distribution near $x$, rather than by the ambient Euclidean dimension. 

Throughout the sharp upper-bound results, we work in the regime $K_0 \log n \le k \le \kappa n$, for sufficiently large $K_0>0$ and $0 < \kappa < c_- r_0^s/8$. The threshold $K_0$ may depend on the desired tail exponent or moment order. The upper bound ensures that the reference scale $r_*(n,k)$ lies within the local regime $r\le r_0$, where Assumption~\ref{ass:lower} applies. The lower bound $k \gtrsim \log n$ is sufficient for our blocking proof: it arises because the proof separates the sample into blocks of length $\asymp \log n$, as is standard in Bernstein-type concentration arguments for geometric mixing sequences \citep{liebscher1996strong,merlevede2009bernstein}. We do not establish that this threshold is necessary. In contrast, the lower bound on the radius does not require $k \gtrsim \log n$. 

\subsection{Non-asymptotic upper bound}

We first establish a tail bound on $R_{n,k}(x)$.
\begin{theorem}[Non-asymptotic tail bound]\label{thm:tail}
Assume Assumptions~\ref{ass:geom} and~\ref{ass:lower}. Fix $A>0$ and define $b_{n,A} := \lceil \frac{A+1}{\gamma}\log n \rceil$, $r_*(n,k) := (\frac{8k}{c_- n})^{1/s}$. There exist constants $c_A,C_A,K_A>0$, depending only on $A$ and $\gamma$, such that, whenever $K_A\log n\le k\le\kappa n$, for every integer $j \ge 0$ satisfying $2^j r_*(n,k) \le r_0$, we have
\[
  \P\big(R_{n,k}(x) > 2^j r_*(n,k)\big)
  \le
  4\exp \Big(-c_A\,2^{js}\,\frac{k}{\log n}\Big)
  + C_A n^{-A},
\]
for all sufficiently large $n$.
\end{theorem}

Theorem~\ref{thm:tail} shows that the $k$-NN radius is concentrated around the canonical geometric scale $(k/n)^{1/s}$. In particular, Theorem~\ref{thm:tail} separates the roles of geometry and dependence: the exponent $1/s$ is determined by the local mass behavior of $\mu$ near $x$, while the dependence structure enters only through the additional $\log n$ factor in the concentration exponent, which is the price of the blocking argument under geometric mixing. Thus, although dependence weakens concentration relative to the i.i.d.\ case, it does not change the underlying nearest-neighbor scale. For $t\ge1$, Theorem~\ref{thm:tail} gives $R_{n,k}(x)\le C((k+t\log n)/n)^{1/s}$ with probability at least $1-4e^{-c_A t}-C_A n^{-A}$, provided this upper radius is at most $r_0$ (after adjusting constants). Choosing $A>p/s$ makes the remainder small enough to obtain the moment bound for every $p>0$.

Integrating the tail bound yields the sharp moment upper bound.

\begin{theorem}[Sharp moment upper bound]\label{thm:moment}
Assume Assumptions~\ref{ass:geom}, \ref{ass:lower}, and \ref{ass:compact}. Let $p>0$. Then there exist constants $K_0,C_{p,s}>0$, depending only on $(p,s,\gamma,c_-,r_0,D)$, such that, for all sufficiently large $n$, whenever $K_0\log n \le k \le \kappa n$, we have $\E\big[R_{n,k}(x)^p\big] \le
  C_{p,s}(k/n)^{p/s}$.
\end{theorem}

If $\mu$ admits a density, then by the Lebesgue differentiation theorem, Assumption~\ref{ass:lower} holds with $s=d$ for $\mu$-almost every interior point at which the density is positive. Consequently, Theorem~\ref{thm:moment} applies for $\mu$-almost all such $x$. Theorem~\ref{thm:moment} is applicable in downstream analyses, since many local learning bounds depend on moments of the random neighborhood. We discuss this further in Section~\ref{sec:discussions}.

\subsection{Matching non-asymptotic lower bound}

To show that the exponent $p/s$ is optimal, we next establish a matching lower bound. Unlike the upper bound, the lower bound is purely local and does not require any dependence condition.

\begin{theorem}[Matching lower bound]\label{thm:lower}
Assume Assumption~\ref{ass:upper}. Let $p>0$. Then whenever $(\frac{k}{2c_+ n})^{1/s} \le r_0$, we have $\E\big[R_{n,k}(x)^p\big] \ge [k/(2c_+ n)]^{p/s}/2$.
\end{theorem}

Theorem~\ref{thm:lower} is conceptually simple but important: if a ball of radius $r$ contains at most order $r^s$ probability mass, then one cannot expect the $k$th nearest neighbor to lie much closer than $(k/n)^{1/s}$. Combining the upper and lower bounds yields the following exact-order characterization.
\begin{theorem}[Exact-order characterization]\label{thm:exact}
Assume Assumptions~\ref{ass:geom}, \ref{ass:lower}, \ref{ass:upper}, and \ref{ass:compact}, with the same exponent $s>0$. Then for every $p>0$ there exist a threshold $K_0$ and constants $0 < c_{p,s} \le C_{p,s} < \infty$, depending only on $(p,s,\gamma,c_-,c_+,r_0,D)$, such that, for all sufficiently large $n$ and all $K_0\log n \le k \le \kappa n$, we have $c_{p,s}(k/n)^{p/s}
  \le
  \E\big[R_{n,k}(x)^p\big]
  \le
  C_{p,s}(k/n)^{p/s}$.
\end{theorem}

Theorem~\ref{thm:exact} shows that the nearest-neighbor radii scale remains unchanged under geometric mixing. The dependence changes constants and concentration behavior, but not the fundamental exponent.




\section{Experiments}

\subsection{Moment convergence rate} \label{exp1:Moment convergence rate}

We test whether the predicted moment rate remains visible across different geometrically $\alpha$-mixing processes and dependence strengths. The implementation compares the i.i.d.\ benchmark with three dependent sequence models: a linear Gaussian state-space model, a hidden Markov model, and a stationary Gaussian-process sequence with a radial basis function (RBF)-like kernel generated by fast Fourier transform (FFT). All observations are marginally transformed to the cube $[0,1]^d$.

We sample 1000 points uniformly from the interior of the cube and report the mean $p$-th moment of the $k$-NN radius. For each $d\in\{1,3,5\}$, we choose $p$ such that $p/d\in\{1,2,3,4,5\}$. We use exponentially growing, dimension-dependent sample sizes $n(d)=100\cdot 2^d\cdot m$ for $m\in\{1,2,4,8,16,32\}$. The number of NNs is chosen as $k=\lfloor n^\beta\rceil$ with $\beta$ on a 20-point grid in $[0.1,0.9]$ for slope fitting, subject to $k/n\le0.01$. The slope is the coefficient $b$ from $\log M = a + b\log(k/n)$, with theoretical slope $p/d$ for the i.i.d.\ and $\alpha$-mixing setting.

\begin{figure}[htbp]
    \centering
    \includegraphics[width=\textwidth]{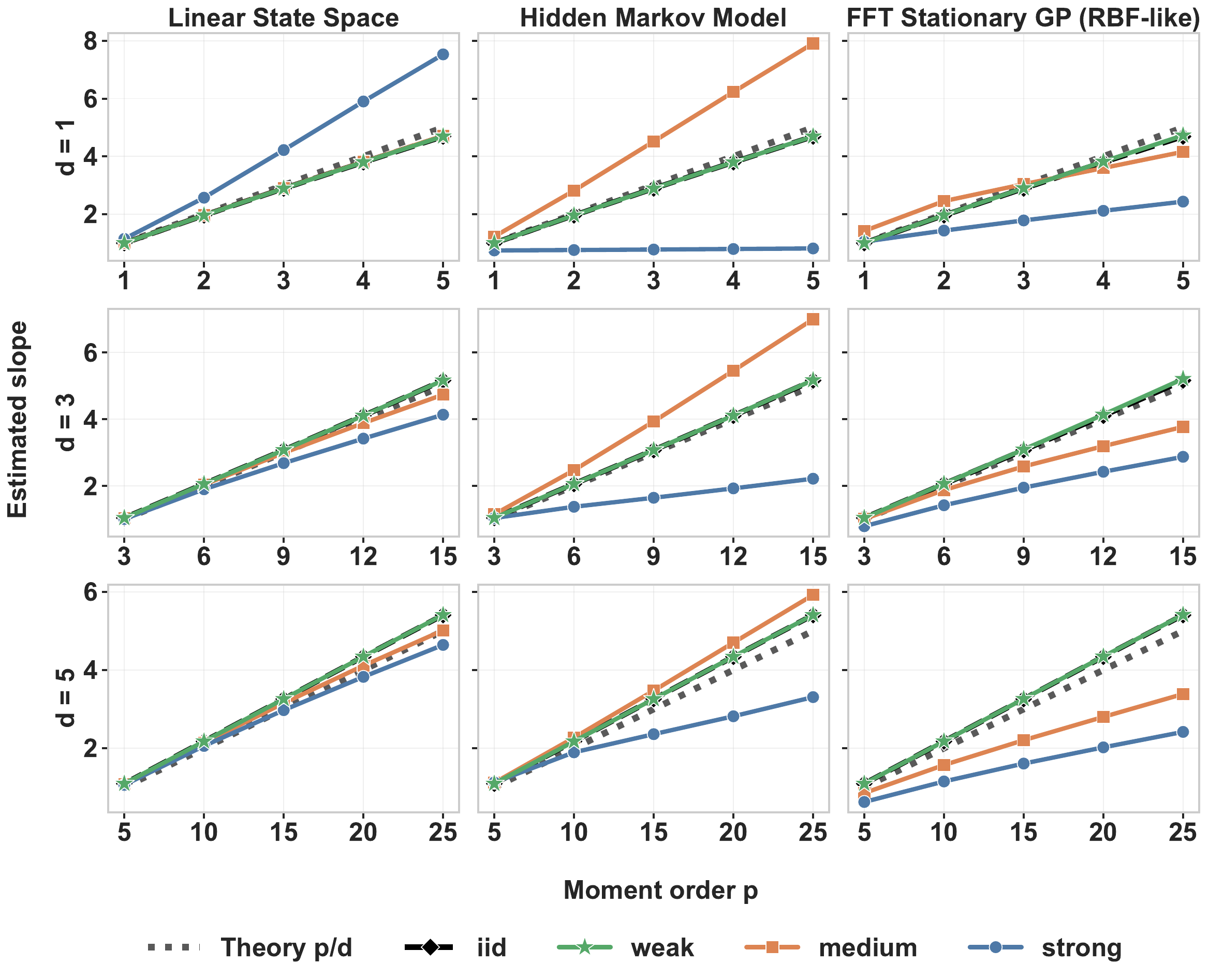}
    \caption{Slope summary for Experiment~1: rows correspond to $d=1,3,5$, columns correspond to the three dependent model classes, and the shared legend compares the theoretical line $p/d$, the i.i.d.\ benchmark, and weak/medium/strong dependence.}
    \label{fig:exp1-slope-grid}
\end{figure}

These generators satisfy Assumption~\ref{ass:geom}; explicit rates are given in Propositions~\ref{prop:mix-lss}, \ref{prop:mix-hmm}, and~\ref{prop:mix-gp}, and the local-mass exponents are checked in Appendix~\ref{app:assumption-empirical}. The departures from $p/d$ in Figure~\ref{fig:exp1-slope-grid} occur for the strong-dependence settings, whose measured effective sample size at $n=6400$ and $d=1$ is at most $90$, while every configuration with $n_{\mathrm{eff}}\gtrsim 2500$ matches the i.i.d.\ slope error (Figure~\ref{fig:diag-neff} and Table~\ref{tab:assumption-diag}).

Across the three data-generating processes (linear state-space, hidden Markov, and FFT Gaussian process), we impose weak/medium/strong dependence via model-specific persistence parameters. Figure~\ref{fig:exp1-slope-grid} shows the expected pattern: weak dependence stays close to the i.i.d.\ baseline and the theoretical rate $p/d$, while medium and strong dependence diverge.

\subsection{Entropy estimation} \label{exp2: Entropy estimation}

This experiment examines the intrinsic-dimension message of the nearest-neighbor theory in a representation-learning setting. We state our results in terms of a local intrinsic dimension $s$. A Kozachenko--Leonenko entropy estimate is built from the NN radii and the corresponding $s$-dimensional volume correction. We use this estimator to illustrate the effect of a representation that preserves the local $s$-dimensional geometry; consistency of the entropy estimator is not implied by our fixed-query radius results. We consider the setup in which observations have ambient dimension $d$, but the true data-generating process has intrinsic dimension $s$.

The experiment generates an $s$-dimensional Gaussian AR(1) latent process, whose stationary marginal differential entropy is available in closed form as $H_s=(s/2)\log(2\pi e)$, and maps it into $20$-dimensional observed data through a linear augmentation. Figure~\ref{fig:exp2-entropy} reports the results for $s\in\{1,3,5\}$ under three temporal-dependence regimes for the sample sequence: i.i.d. ($\rho=0$), weak AR(1) dependence ($\rho=0.3$), and strong AR(1) dependence ($\rho=0.6$), using $1000$ replications and number of NNs $k=\max\{2,\lceil n^{0.1}\rceil\}$. The oracle estimator on the latent $s$-dimensional data and the estimator based on the top $s$ coordinates of a $5$-component principal component analysis (PCA) representation both move toward the true entropy as $n$ grows. In this noiseless design the embedding is an isometry, and the top $s$ PCA coordinates preserve pairwise distances once the centered sample spans the signal subspace. Thus the two entropy estimates agree up to numerical precision; this is not a guarantee for arbitrary learned representations. The schedule $k=\lceil n^{0.1}\rceil$ eventually exceeds any fixed multiple of $\log n$, although finite sample sizes need not meet the theorem's threshold. Implementation details are given in Appendix~\ref{app: Experiment 2 implementation details}.

\begin{figure}[H]
    \centering
    \includegraphics[width=1\textwidth]{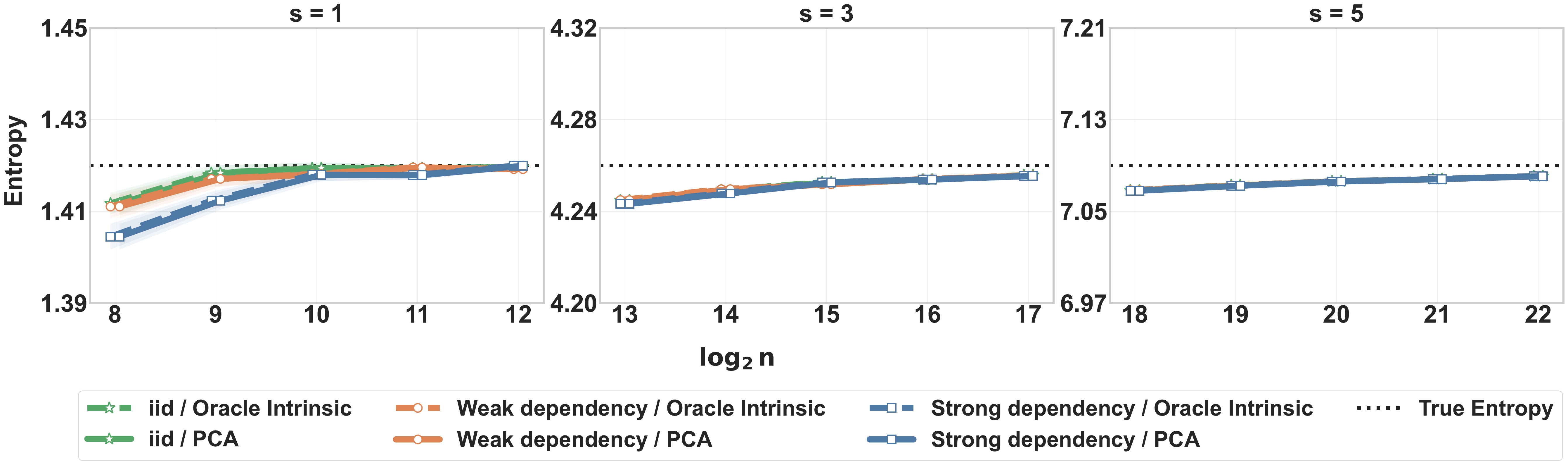}
    \caption{Experiment~2 entropy estimation. The dotted horizontal line is the true entropy. Colors encode the dependence level and markers distinguish the oracle intrinsic estimate from the PCA estimate.}
    \label{fig:exp2-entropy}
\end{figure}

The same design satisfies the assumptions with intrinsic exponent $s$ rather than ambient dimension $20$. Remark~\ref{rem:exp2} shows that the latent AR(1) is geometrically $\alpha$-mixing, that the fixed orthonormal embedding preserves both the mixing rate and the local ball mass of order $s$, and that the Gaussian marginal is covered by the moment bound in Proposition~\ref{prop:moment-noncompact}. Figure~\ref{fig:diag-exp2} shows the empirical ball mass in $\R^{20}$ lying on the exact $s$-dimensional curves, with fitted slopes about $2.84$ and $4.81$--$4.85$ for targets $s=3$ and $s=5$.

\subsection{Real-data evaluation}\label{exp3: Real-data_evaluation}


In this section, we complement the synthetic experiments with real-world time-series benchmarks. We compare two baselines: a raw $k$-nearest-neighbors (Raw-$k$NN) method operating directly in standardized observation space, and a PCA $k$-nearest-neighbors (PCA-$k$NN) method that first projects each input onto a lower-dimensional representation using PCA before applying nearest-neighbor prediction. We evaluate Raw-$k$NN and PCA-$k$NN on three real-data tasks: long-horizon forecasting, short-horizon forecasting, and classification. The datasets are from the Time Series Pile \citep{goswami2024moment}, a large public time-series benchmark collection that includes the Informer long-horizon forecasting datasets \citep{zhou2021informer}, the Monash Time Series Forecasting Archive \citep{godahewa2021monash}, and the UCR Time Series Archive \citep{dau2019ucr}. We follow the benchmark setup and compare our results against MOMENT, a family of pre-trained transformer models for general-purpose time-series analysis \citep{goswami2024moment}. Forecasting predicts future values from past windows, while classification assigns a label to an entire series. In both nearest-neighbor models, the input features are standardized before nearest-neighbor prediction. We tune the number of neighbors $k \in \{1,3,5,7,9,15,21,31\}$, the neighbor weighting scheme (uniform/distance), and the distance metric (Manhattan/Euclidean). For PCA-$k$NN, we additionally tune the embedding dimension over $\{8,16,32,48,64\}$. Dataset definitions, metrics, preprocessing, train/validation/test splits, and tuning protocols for all three tasks are provided in Appendix~\ref{app:real-data-implementation}.


Table~\ref{tab:real_data_results} summarizes average performance across datasets for long-horizon forecasting, short-horizon forecasting, and classification. Figure~\ref{fig:long_horizon_mse} reports long-horizon forecasting mean squared error (MSE) separately for each dataset and forecast horizon. Figure~\ref{fig:short_horizon_vs_classification} provides complementary views of short-horizon forecasting and classification: its left panel shows short-horizon forecasting symmetric mean absolute percentage error (sMAPE) under both the within-group and source-to-target transfer settings, while its right panel shows the distribution of classification accuracy with MOMENT summary statistics included as reference lines. Together, these results are supportive in two main ways. First, Raw-$k$NN and PCA-$k$NN remain informative across all three task families. In long-horizon forecasting, the two methods achieve average MSE $0.628$ and $0.627$, respectively, compared with $0.617$ for MOMENT. In short-horizon forecasting, both methods produce nontrivial predictions in both the within-group and source-to-target transfer settings, although errors increase substantially in the transfer setting. In classification, both methods achieve mean accuracy $0.748$, compared with $0.794$ for MOMENT. These findings indicate that nearest-neighbor structure remains informative on real dependent data. Second, Raw-$k$NN and PCA-$k$NN achieve similar performance across tasks, consistent with dimension reduction preserving information useful for nearest-neighbor prediction. Their relative advantage varies across datasets, but neither method consistently dominates the other. 


Although MOMENT generally outperforms the Raw-$k$NN and PCA-$k$NN baselines, this does not contradict the theory. Both nearest-neighbor methods remain effective on real dependent data, and PCA-$k$NN performs comparably to Raw-$k$NN even after reducing the input dimension. This suggests that dimension reduction can preserve the local structure needed for neighborhood-based prediction. Appendix~\ref{app:real-data-diagnostics} directly examines radius scaling and local mass: fitted radius slopes are positive across the reported datasets, with 95\% query-block-bootstrap intervals bounded away from zero. These finite-range diagnostics do not establish the exact exponent $1/s$ or verify $\alpha$-mixing; the intervals condition on the reference data and fitted representations.

\begin{table}[htbp]
  \caption{Comparison of Raw-$k$NN, PCA-$k$NN, and MOMENT on real-data tasks under a common benchmark setup. Performance is averaged over datasets within each task; MOMENT results are taken from \cite{goswami2024moment}.}
  \label{tab:real_data_results}
  \centering
  \footnotesize
  \begin{tabular}{@{}lcccc@{}}
    \toprule
    Task & Metric & Raw-$k$NN & PCA-$k$NN & MOMENT \\
    \midrule
    Long-horizon forecasting & MSE & 0.628 & 0.627 & 0.617 \\
    Classification & Accuracy & 0.748 & 0.748 & 0.794 \\
    Short-horizon forecasting (within-group) & sMAPE & 16.199 & 16.488 & -- \\
    Short-horizon forecasting (transfer) & sMAPE & 27.210 & 27.369 & 14.525 \\
    \bottomrule
  \end{tabular}
\end{table}

\begin{figure}[htbp]
    \centering
    \includegraphics[width=0.92\linewidth]{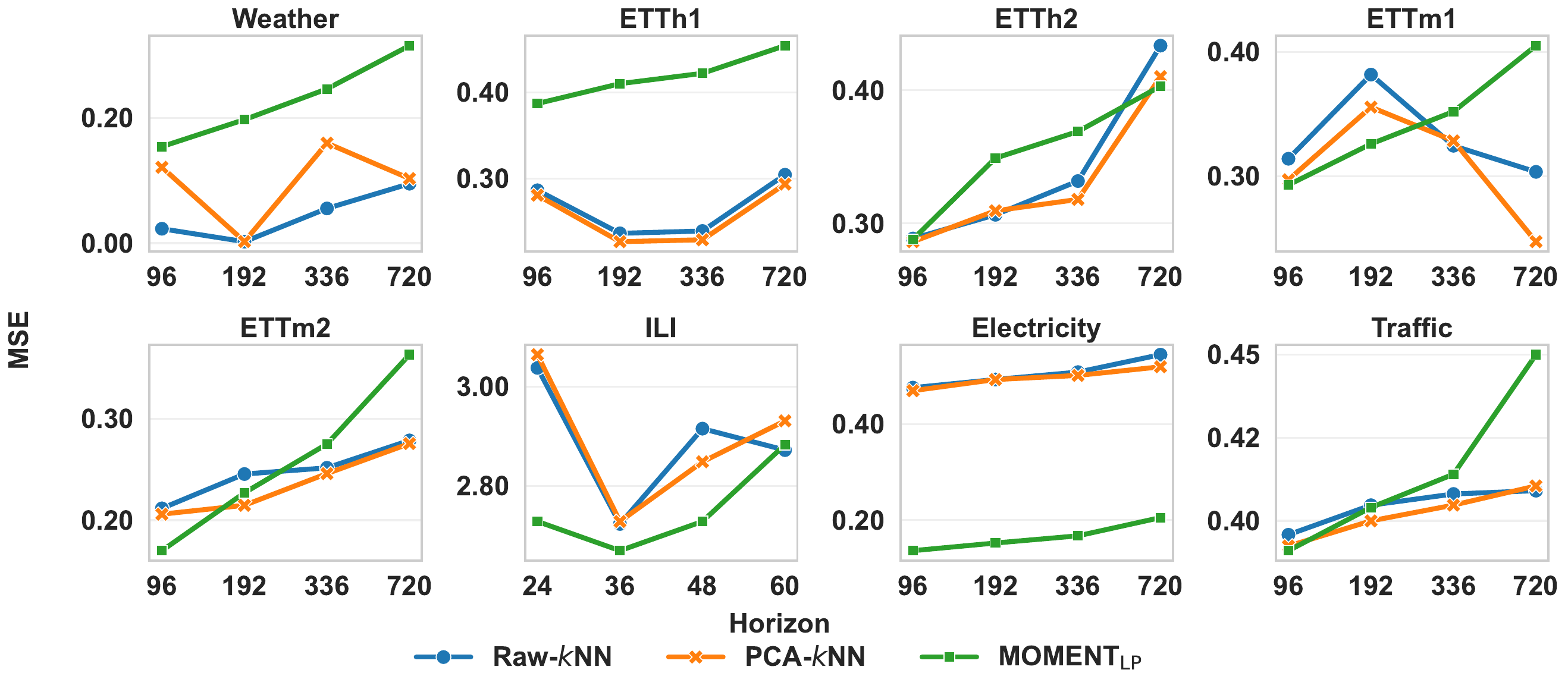}
    \caption{Comparison of long-horizon forecasting using test MSE across datasets and forecast horizons. It compares Raw-$k$NN and PCA-$k$NN with MOMENT across benchmark settings.}
    \label{fig:long_horizon_mse}
\end{figure}

\begin{figure}[htbp]
    \centering
    \begin{subfigure}[t]{0.48\linewidth}
        \centering
        \includegraphics[width=\linewidth]{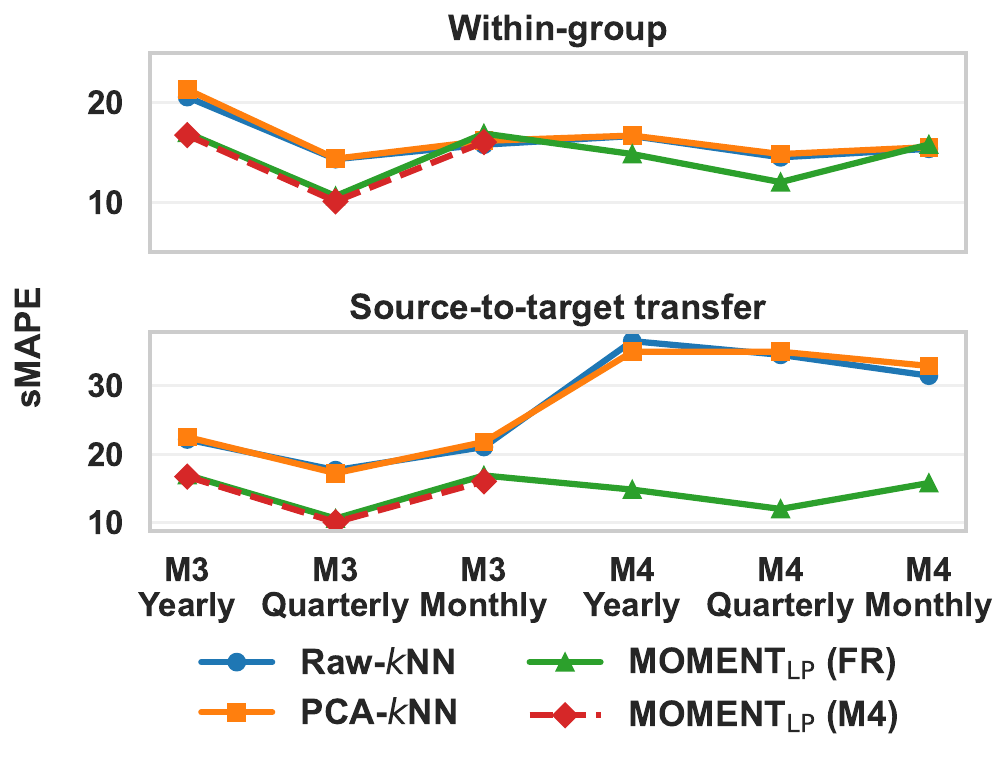}
        \label{fig:short_horizon_comparison}
    \end{subfigure}
    \hfill
    \begin{subfigure}[t]{0.51\linewidth}
        \centering
        \includegraphics[width=\linewidth]{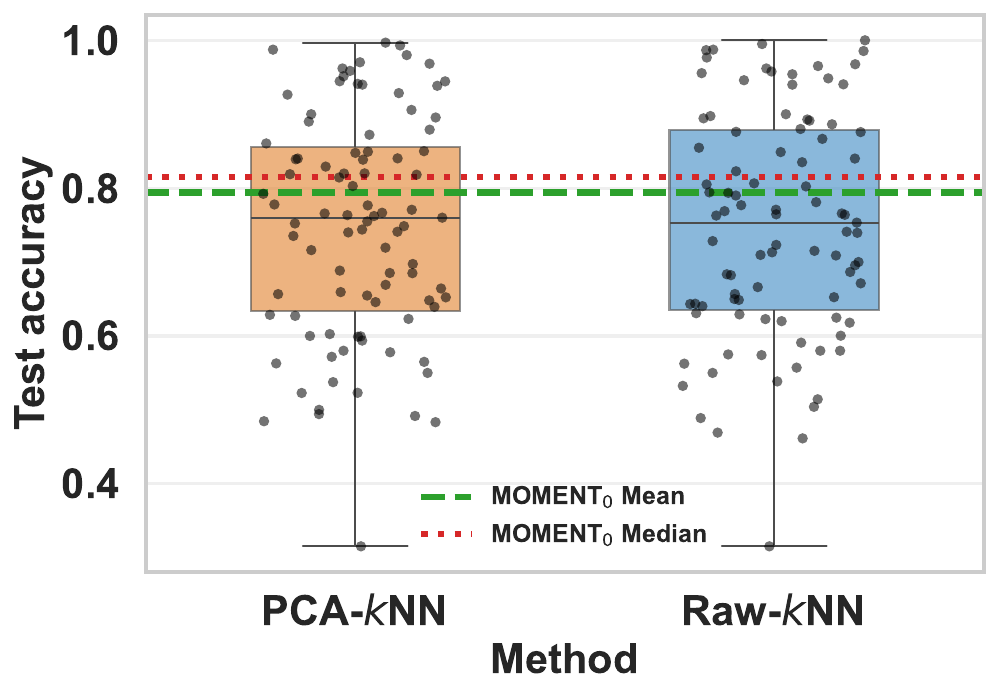}
        \label{fig:classification_boxplot}
    \end{subfigure}
    \caption{Comparisons on short-horizon forecasting and classification. The left panel compares forecasting under within-group and source-to-target transfer settings. The right panel shows classification accuracy across UCR datasets, with MOMENT summary statistics shown as reference lines.}
  \label{fig:short_horizon_vs_classification}
\end{figure}

\section{Discussions and limitations}\label{sec:discussions}

\paragraph*{Constants versus exponents.}
One main message of the theory is that dependence and local geometry play qualitatively different roles. The local exponent $s$ determines the fundamental scale $(k/n)^{1/s}$, whereas the dependence structure determines constants and concentration quality. This distinction is important both conceptually and empirically: it suggests that weak dependence should affect finite-sample behavior without changing the asymptotic behavior.

\paragraph*{On uniform convergence.} We study nearest-neighbor radii at a fixed query point $x$. Uniform control requires a lower-mass condition with common $(c_-,s,r_0)$ over a deterministic query set $Q$, together with control of its covering numbers. For an $\varepsilon$-net of $Q$ with $M$ centers in $Q$, a union bound and the $1$-Lipschitz property of $x\mapsto R_{n,k}(x)$ give
\[
\P\!\left(\sup_{x\in Q}R_{n,k}(x)>2^j r_*+\varepsilon\right)
\le M\left[4e^{-c_A2^{js}k/\log n}+C_A n^{-A}\right],
\qquad 2^j r_*\le r_0.
\]
Thus covering entropy enters the concentration bound, and $A$ must also control the accumulated remainder. The same union bound applies to a fixed finite set of queries. Sample-dependent queries require a uniform event covering their possible locations; they cannot simply be substituted into a fixed-query bound. A general uniform theory for downstream estimators requires further assumptions and is left to future work \citep{yu1994rates,jiang2019non,gyorfi2002distribution}.

\paragraph*{Time-series representations.} The synthetic proofs cover Gaussian state-space, hidden Markov, and GP--RBF models (Appendix~\ref{app:assumption-verification}); stationary causal ARMA models are also covered under the regularity conditions of Mokkadem~\citep{mokkadem1988mixing}. A fixed measurable frame-level encoder cannot increase $\alpha$-mixing coefficients. If it is bi-Lipschitz on the support, it also preserves the local mass exponent, with changed constants. Neither conclusion automatically extends to an encoder fitted on the same dependent sample, and autocorrelation decay alone does not verify strong mixing.

\paragraph*{Implications for learning methods.}
More broadly, we provide a geometric control principle for a wide class of learning methods under dependent sampling. Since the $k$-NN radius determines the effective neighborhood size around a query point, our results imply that, under the stated assumptions, dependent observations behave as if the relevant scale were still $(k/n)^{1/s}$ up to constants and a dependence-induced concentration penalty. In particular, whenever the performance of a method is driven by the size of adaptive neighborhoods---as in $k$-NN regression, local averaging, adaptive bandwidth selection, graph construction, or neighborhood-based manifold methods---our results can be used as an ingredient to analyze their performance on dependent data. Detailed development of these implications for nonparametric and graph learning algorithms is left for future work.

\clearpage
\section*{Acknowledgments}
We thank Fang Han for providing part of the motivation for this work and for insightful comments and discussions. We also thank the anonymous reviewers, Area Chairs, and Senior Area Chairs for their valuable feedback and suggestions.
\section*{Funding and competing interests}
The authors received no external funding or third-party support for this work. The authors declare no competing interests.
\bibliographystyle{plain}
\bibliography{AMS}

\newpage
\appendix

\section{Proofs}\label{sec:proof}

\begin{table}[t]
\centering
\caption{Notation used in the results and proofs.}
\label{tab:notation}
\begin{tabular}{@{}ll@{}}
\toprule
Symbol & Meaning \\
\midrule
$X_{n,i}$, $\mu$ & Observations and their common marginal law \\
$R_{n,k}(x)$, $N_n(x,r)$ & $k$-NN radius and number of observations in $B(x,r)$ \\
$r_x$, $p_r$ & Distance to ${\rm supp}(\mu)$ and ball mass $\mu(B(x,r))$ \\
$s$, $c_\pm$, $r_0$ & Local mass exponent, constants, and radius of validity \\
$Y_i(r)$, $Z_{n,i}$ & Ball indicator and centered indicator, respectively \\
$b_{n,A}$, $W_u(r)$, $V_u$ & Block length, kept-block count, and centered block count \\
\bottomrule
\end{tabular}
\end{table}

The proofs use the equivalence $R_{n,k}(x)>r\iff N_n(x,r)<k$. For almost sure convergence, a Bernstein bound at a fixed radius gives summable probabilities. For the finite-sample upper bounds, we retain alternating blocks of length $O(\log n)$, control their mixing coefficients, and integrate the resulting tail bound. The matching lower bound uses only the common marginal law and Markov's inequality.

We begin with a Bernstein-type inequality for \(\alpha\)-mixing sequences. The alternative inequality of Merlev\`ede, Peligrad, and Rio~\citep[Theorem~2]{merlevede2009bernstein} involves a long-run covariance sum and a $(\log n)^2$ term. Lemma~\ref{lemma1} is convenient here because the kept-block variance bound retains the small-ball probability and the block length makes the mixing remainder arbitrarily small. Different inequalities may alter constants and logarithmic thresholds, but the matching lower bound shows that the geometric exponent $p/s$ cannot be improved under the stated local-mass assumptions.

\begin{lemma}[Theorem 5 of \cite{rio1995functional}; Theorem 2.1 of \cite{liebscher1996strong}]\label{lemma1}
Let \([Z_i]_{i=1}^n\) be a centered sequence of random variables such that
\[
\E[Z_i]=0
\qquad\text{and}\qquad
|Z_i|\le S(n)<\infty,
\qquad i\in \{1,\ldots,n\}.
\]
Then, for any \(n\ge 1\), any \(m\in \{1,\ldots,n\}\), and any \(\epsilon>4mS(n)\),
\[
\P\Big(\Big|\sum_{i=1}^n Z_i\Big|>\epsilon\Big)
\le
4\exp \left(
-\frac{\epsilon^2}{64\frac{n}{m}\sigma_{n,m}^2+\frac83\,\epsilon m S(n)}
\right)
+
4\frac{n}{m}\alpha_m,
\]
where
\[
\sigma_{n,m}^2
:=
\sup_{0\le j\le n-1}
\E\Big[\Big(\sum_{i=j+1}^{(j+m)\wedge n} Z_i\Big)^2\Big].
\]
\end{lemma}

\subsection{Proof of Theorem~\ref{thm:as}}

\begin{proof}

Fix \(x\in \R^d\). Since \(\mathcal X_{(k)}^n(x)\in {\rm supp}(\mu)\) almost surely and $r_x =\inf\{\|y-x\|: y\in {\rm supp}(\mu)\}$, we have
\[
R_{n,k}(x) \ge r_x
\qquad\text{a.s. for every }n.
\]
Thus, it remains to prove that
\[
\limsup_{n\to\infty}R_{n,k}(x) \le r_x
\qquad\text{a.s.}
\]

Fix \(\varepsilon>0\), and define
\[
p_x := \mu(B(x,r_x+\varepsilon))
= \P\big(\|X_{n,i}-x\|\le r_x+\varepsilon\big).
\]
Since the interior of \(B(x,r_x+\varepsilon)\) intersects \({\rm supp}(\mu)\) by definition of \(r_x\), we have \(p_x>0\).

We observe that the event \(R_{n,k}(x)>r_x+\varepsilon\) occurs if and only if fewer than \(k\) observations in the $n$-th row fall inside \(B(x,r_x+\varepsilon)\). Then we have
\[
\P\big(R_{n,k}(x)>r_x+\varepsilon\big)
=
\P\Big(\sum_{i=1}^n \mathbf 1\{\|X_{n,i}-x\|\le r_x+\varepsilon\}<k\Big).
\]
We set
\[
Z_{n,i}
:=
\mathbf 1\{\|X_{n,i}-x\|\le r_x+\varepsilon\}
-
\P\big(\|X_{n,i}-x\|\le r_x+\varepsilon\big)
=
\mathbf 1\{\|X_{n,i}-x\|\le r_x+\varepsilon\}-p_x.
\]
Then \(\E[Z_{n,i}]=0\) and \(|Z_{n,i}|\le 1\). Since \(Z_{n,i}\) is a measurable function of \(X_{n,i}\), the sequence \([Z_{n,i}]_{i=1}^n\) is again \(\alpha\)-mixing with mixing coefficients no greater than those of \([X_{n,i}]_{i=1}^n\). Moreover, for any \(m\in\{1,\dots,n\}\) and any \(0\le j\le n-1\), the covariance inequality \citep[Corollary of Lemma 2.1]{davydov1968convergence} yields
\begin{align*}
\E\Big[\Big(\sum_{i=j+1}^{(j+m)\wedge n} Z_{n,i}\Big)^2\Big]
&=
\Var\Big(\sum_{i=j+1}^{(j+m)\wedge n} Z_{n,i}\Big) \\
&\le \sum_{i=j+1}^{(j+m)\wedge n}\Var(Z_{n,i})
+2\sum_{i=j+1}^{(j+m)\wedge n}\sum_{k=i+1}^{n}|\Cov(Z_{n,i},Z_{n,k})| \\
&\le
m + 24\sum_{i=j+1}^{(j+m)\wedge n}\sum_{k=i+1}^{n}\alpha_{k-i}
\|Z_{n,i}\|_\infty \|Z_{n,k}\|_\infty \\
&\le
m + 24\sum_{i=j+1}^{(j+m)\wedge n}\sum_{k=i+1}^{n}\alpha_{k-i} \\
&\le
m + 24m\sum_{\ell=1}^\infty \alpha_\ell.
\end{align*}
Since \(\sum_{\ell=1}^\infty \alpha_\ell<\infty\), we set $M = 1+24\sum_{\ell=1}^\infty \alpha_\ell <\infty$, then
\[
\sigma_{n,m}^2
:=
\sup_{0\le j\le n-1}
\E\Big[\Big(\sum_{i=j+1}^{(j+m)\wedge n} Z_{n,i}\Big)^2\Big]
\le Mm.
\]
Since \(k/n\to 0\) and \(p_x>0\), for all sufficiently large $n$, we have $k-np_x \le -\frac{np_x}{2}$. Then
\begin{align*}
\P\big(R_{n,k}(x)>r_x+\varepsilon\big)
&=
\P\Big(\sum_{i=1}^n Z_{n,i}<k-np_x\Big) \\
&\le
\P\Big(\Big|\sum_{i=1}^n Z_{n,i}\Big|>np_x-k\Big) \\
&\le
\P\Big(\Big|\sum_{i=1}^n Z_{n,i}\Big|>\frac{np_x}{2}\Big).
\end{align*}

We now apply Lemma~\ref{lemma1} with $m_n = \lfloor \frac{3}{64}\frac{np_x}{\log n}\rfloor$. Then \(m_n\to\infty\) and \(m_n/n \to 0\). Therefore for all sufficiently large $n$, we have $64M \le \frac{4}{3}p_x m_n$ and $
\frac{np_x}{2}>4m_n$, and then the conditions of Lemma~\ref{lemma1} are satisfied with \(S(n)=1\) and \(\epsilon_n=np_x/2\). Using \(\sigma_{n,m_n}^2\le M m_n\), we obtain
\begin{align*}
\P\Big(\Big|\sum_{i=1}^n Z_{n,i}\Big|>\frac{np_x}{2}\Big)
&\le
4\exp \left[
-\frac{(np_x/2)^2}{64(n/m_n)\sigma_{n,m_n}^2 + \frac83 (np_x/2)m_n}
\right]
+4\frac{n}{m_n}\alpha_{m_n} \\
&\le
4\exp \left[
-\frac{n^2p_x^2/4}{64Mn + \frac43 np_x m_n}
\right]
+4\frac{n}{m_n}\alpha_{m_n} \\
&\le
4\exp \left[
-\frac{n^2p_x^2/4}{\frac83 np_x m_n}
\right]
+4\frac{n}{m_n}\alpha_{m_n} \\
&=
4\exp \left(-\frac{3}{32}\frac{np_x}{m_n}\right)
+4\frac{n}{m_n}\alpha_{m_n}.
\end{align*}
By the choice of \(m_n\), the exponential term is \(O(n^{-2})\).

By the assumed bound $\alpha_m \lesssim m^{-\gamma}$ for some $\gamma>1$,
\[
\frac{n}{m_n}\alpha_{m_n} \lesssim \frac{n}{m_n^{1+\gamma}} \lesssim \frac{(\log n)^{1+\gamma}}{n^{\gamma}},
\]
which is summable in $n$. Then
\[
\sum_{n=1}^\infty
\P\big(R_{n,k}(x)>r_x+\varepsilon\big)
<\infty.
\]
By the Borel--Cantelli lemma,
\[
\P\big(R_{n,k}(x)>r_x+\varepsilon\ \text{i.o.}\big)=0.
\]
That is, for every fixed \(\varepsilon>0\),
\[
R_{n,k}(x)\le r_x+\varepsilon
\qquad\text{eventually a.s.}
\]

Let \(\varepsilon=1/j\), \(j\in\mathbb N\), and intersect the resulting probability-one events. Then we obtain
\[
\limsup_{n\to\infty}R_{n,k}(x)\le r_x
\qquad\text{a.s.}
\]
Together with the lower bound \(R_{n,k}(x)\ge r_x\) a.s. for every $n$, we obtain
\[
R_{n,k}(x) \to r_x
\qquad\text{almost surely}.
\]
This completes the proof.

\end{proof}

\subsection{Proof of Theorem~\ref{thm:tail}}

\begin{proof}

Fix $A>0$, set $b_n=b_{n,A}:=\lceil((A+1)/\gamma)\log n\rceil$, and suppose $K_A\log n\le k\le\kappa n$, with $K_A$ chosen below. Fix \(r>0\), and define
\[
Y_i(r):=\mathbf 1\{\|X_{n,i}-x\|\le r\},
\qquad
p_r:=\E[Y_i(r)]=\mu(B(x,r)).
\]
Let
\[
N_n(x,r):=\sum_{i=1}^n Y_i(r).
\]
By the definition of the \(k\)-nearest-neighbor radius,
\[
R_{n,k}(x)>r
\iff N_n(x,r)<k.
\]
We now use a blocking argument. Partition the row into alternating kept and discarded blocks of length \(b_n\). Set $M_n=\lfloor \frac{n}{2b_n}\rfloor$, and for \(u=1,\dots,M_n\), define the kept block
\[
I_u:=\{(2u-2)b_n+1,\dots,(2u-1)b_n\}.
\]
Let
\[
W_u(r):=\sum_{i\in I_u}Y_i(r),
\qquad
S_n^{\rm keep}(r):=\sum_{u=1}^{M_n}W_u(r).
\]
Since \(S_n^{\rm keep}(r)\le N_n(x,r)\), we obtain
\[
\P(R_{n,k}(x)>r)
=
\P(N_n(x,r)<k)
\le
\P(S_n^{\rm keep}(r)<k).
\]

Consider the radius
\[
r_j=2^j r_*(n,k)
=
2^j\Big(\frac{8k}{c_- n}\Big)^{1/s},
\qquad j\ge 0,
\]
and assume \(r_j\le r_0\). By Assumption~\ref{ass:lower},
\[
p_{r_j}
=
\mu(B(x,r_j))
\ge c_- r_j^s
=
8\cdot2^{js}\frac{k}{n}.
\]
Since $b_n/n \to 0$ and $M_n b_n \ge n/4$ for all sufficiently large $n$,
\[
\mu_j=\E[S_n^{\rm keep}(r_j)]
=
M_n b_n p_{r_j}
\ge n\,p_{r_j}/4
\ge
2^{js+1}k.
\]
Therefore, on the event \(\{S_n^{\rm keep}(r_j)<k\}\), since $\mu_j\ge 2k$,
\[
\mu_j-S_n^{\rm keep}(r_j)>\mu_j-k\ge \mu_j/2.
\]
Then
\[
\{S_n^{\rm keep}(r_j)<k\}
\subseteq
\left\{
\left|S_n^{\rm keep}(r_j)-\mu_j\right|>\frac{\mu_j}{2}
\right\}.
\]
Therefore, if we define
\[
V_u:=W_u(r_j)-\E[W_u(r_j)],
\qquad u=1,\dots,M_n,
\]
we have
\[
\P(R_{n,k}(x)>r_j)
\le
\P\Big(\Big|\sum_{u=1}^{M_n}V_u\Big|>\frac{\mu_j}{2}\Big).
\]

Next we apply Lemma~\ref{lemma1} for the block sequence \([V_u]_{u=1}^{M_n}\). Each \(V_u\) is measurable with respect to the coordinates in the kept block \(I_u\). For $1\le\ell<M_n$ and $1\le u\le M_n-\ell$, the finite-index inclusions
\[
\begin{aligned}
\sigma(V_1,\dots,V_u)&\subseteq {}^n\cF_1^{(2u-1)b_n},\\
\sigma(V_{u+\ell},\dots,V_{M_n})&\subseteq {}^n\cF_{(2(u+\ell)-2)b_n+1}^{n}
\end{aligned}
\]
show that the corresponding past and future are separated by $(2\ell-1)b_n+1$ indices. Thus the mixing coefficients of the kept-block sequence satisfy
\[
\bar\alpha_\ell\le\alpha_{(2\ell-1)b_n+1}\le\alpha_{\ell b_n},
\qquad 1\le\ell<M_n,
\]
using the monotonicity of the original mixing coefficients. In particular, $\bar\alpha_1\le \alpha_{b_n}\le e^{-\gamma b_n}\le n^{-(A+1)}$ by the choice of $b_n$. Also, $0\le W_u(r_j)\le b_n$, so $|V_u| = |W_u(r_j)-\E W_u(r_j)| \le b_n$. Thus the boundedness parameter in Lemma~\ref{lemma1} is $S(M_n)=b_n$.

To bound the variance term \(\sigma^2_{M_n,1}\), since \(W_u(r_j)\) is a sum of \(b_n\) Bernoulli indicators, we have $W_u(r_j)^2\le b_n W_u(r_j)$, and therefore
\[
\Var(V_u)
\le \E[W_u(r_j)^2]
\le b_n\,\E[W_u(r_j)]
=
b_n^2 p_{r_j}.
\]
Then for the block sequence with block parameter \(m=1\),
\[
\sigma^2_{M_n,1}
=
\sup_{0\le v\le M_n-1}\E[V_{v+1}^2]
\le
b_n^2 p_{r_j}.
\]

We now apply Lemma~\ref{lemma1} to \([V_u]_{u=1}^{M_n}\) with \(m=1\) and threshold $\varepsilon=\mu_j/2=M_n b_n p_{r_j}/2$. Since $\varepsilon\ge k\ge K_A\log n$, and \(b_n\asymp \log n\), we may choose \(K_A\) sufficiently large so that $\varepsilon>4b_n$ for all sufficiently large $n$. Then by Lemma~\ref{lemma1},
\begin{align*}
\P\Big(\Big|\sum_{u=1}^{M_n}V_u\Big|>\frac{\mu_j}{2}\Big)
&\le
4\exp \left(
-\frac{(\mu_j/2)^2}
{64 M_n \sigma_{M_n,1}^2 + \frac83(\mu_j/2)b_n}
\right)
+4M_n\bar\alpha_1 \\
&\le
4\exp \left(
-\frac{(\mu_j/2)^2}
{64 M_n b_n^2 p_{r_j} + \frac43 \mu_j b_n}
\right)
+4M_n\bar\alpha_1.
\end{align*}
Using \(\mu_j=M_n b_n p_{r_j}\), we obtain $\frac43 \mu_j b_n
=
\frac43 M_n b_n^2 p_{r_j}$, so the denominator is bounded by
\[
\left(64+\frac43\right)M_n b_n^2 p_{r_j}
\le 66\,M_n b_n^2 p_{r_j},
\]
and then
\begin{align*}
\frac{(\mu_j/2)^2}
{64 M_n b_n^2 p_{r_j} + \frac43 \mu_j b_n}
&\ge
c\frac{\mu_j^2}{M_n b_n^2 p_{r_j}} =
c\frac{(M_n b_n p_{r_j})^2}{M_n b_n^2 p_{r_j}} =
c M_n p_{r_j}.
\end{align*}
Since \(M_n\ge n/(4b_n)\) for all large $n$,
\[
M_n p_{r_j}
\ge
\frac{n}{4b_n}\cdot 8\cdot2^{js}\frac{k}{n}
=
2\cdot2^{js}\frac{k}{b_n}.
\]
Therefore, for some universal constant \(c_0>0\),
\[
\P(R_{n,k}(x)>r_j)
\le
4\exp \left(-c_0 2^{js}\frac{k}{b_n}\right)
+
4M_n\bar\alpha_1.
\]

Finally, using \(b_n\asymp \log n\), \(\bar\alpha_1\le n^{-(A+1)}\), and \(M_n\le n/(2b_n)\), we obtain $4M_n\bar\alpha_1 \le C_A n^{-A}$, and then
\[
\P(R_{n,k}(x)>r_j)
\le
4\exp \left(-c_A\,2^{js}\frac{k}{\log n}\right)
+
C_A n^{-A},
\]
where $c_A,C_A,K_A$ depend only on $A$ and $\gamma$. Compact support was not used. This completes the proof.
\end{proof}

\subsection{Proof of Theorem~\ref{thm:moment}}

\begin{proof}

Fix $A>p/s$ and choose $K_0\ge K_A$ as in Theorem~\ref{thm:tail}; thus $K_0$ may depend on $p$. Let $$r_* = \Big(\frac{8k}{c_- n}\Big)^{1/s}$$ and choose \(J_0\) maximal such that $2^{J_0}r_* \le r_0$. Assumption~\ref{ass:lower} implies $x\in{\rm supp}(\mu)$. Since the support of \(\mu\) has diameter at most \(D\), we have $R_{n,k}(x)\le D$ almost surely.

We decompose the moment over
\[
\{R_{n,k}(x)\le r_*\},
\quad
\{2^j r_* < R_{n,k}(x)\le 2^{j+1}r_*\},\ 0\le j\le J_0-1,
\quad
\{R_{n,k}(x)>2^{J_0}r_*\},
\]
and we obtain
\begin{align*}
\E[R_{n,k}(x)^p]
&\le r_*^p
+ \sum_{j=0}^{J_0-1}(2^{j+1}r_*)^p\,\P(R_{n,k}(x)>2^j r_*) + D^p\,\P(R_{n,k}(x)>2^{J_0}r_*).
\end{align*}

We now apply Theorem~\ref{thm:tail}. For every \(0\le j\le J_0\),
\[
\P(R_{n,k}(x)>2^j r_*)
\le
4\exp \Big(-c_A\,2^{js}\frac{k}{\log n}\Big)+Cn^{-A}.
\]
At $j=J_0$, maximality of $J_0$ gives
\[
2^{J_0s}k=\frac{c_-n}{8}(2^{J_0}r_*)^s
>\frac{c_-n}{8}(r_0/2)^s.
\]
Consequently, the exponential term at $J_0$ is at most $4e^{-cn/\log n}\le Cn^{-A}$ for all sufficiently large $n$. We therefore obtain
\begin{align*}
\E[R_{n,k}(x)^p]
&\le r_*^p
+ C r_*^p \sum_{j=0}^{J_0-1}2^{jp}
\exp \Big(-c_A\,2^{js}\frac{k}{\log n}\Big) + C r_*^p n^{-A}\sum_{j=0}^{J_0-1}2^{jp}
+ C D^p n^{-A}.
\end{align*}

We now bound the three terms on the right-hand side. For the first term, since \(k\ge K_0\log n\), the quantity \(k/\log n\) is bounded below by a positive constant. Therefore
\[
\sum_{j=0}^{J_0-1}2^{jp}
\exp \Big(-c_A\,2^{js}\frac{k}{\log n}\Big)
\le
\sum_{j\ge 0}2^{jp}e^{-c_1 2^{js}}
<\infty
\]
for some constant \(c_1>0\). For the second term, since \(2^{J_0}r_*\le r_0<2^{J_0+1}r_*\), we have $2^{J_0}\asymp r_*^{-1}$, where the constants depend only on \(r_0\). Then
\[
r_*^p\sum_{j=0}^{J_0-1}2^{jp}
\le
C r_*^p 2^{J_0 p}
\le C,
\]
and then
\[
C r_*^p n^{-A}\sum_{j=0}^{J_0-1}2^{jp}
\le C n^{-A}.
\]
Since $A>p/s$ and $k\ge1$, we have $n^{-A}\le n^{-p/s}\le(k/n)^{p/s}\asymp r_*^p$. The third term \(C D^p n^{-A}\) is bounded in the same way. Combining the above estimates yields
\[
\E[R_{n,k}(x)^p]
\le C r_*^p + C n^{-A}
\le C r_*^p.
\]
Since
\[
r_*^p
=
\Big(\frac{8k}{c_- n}\Big)^{p/s}
\asymp
\Big(\frac{k}{n}\Big)^{p/s},
\]
we obtain
\[
\E[R_{n,k}(x)^p]
\le
C\Big(\frac{k}{n}\Big)^{p/s}.
\]
This completes the proof.
\end{proof}

\subsection{Proof of Theorem~\ref{thm:lower}}

\begin{proof}
Let $$r^- = \Big(\frac{k}{2c_+ n}\Big)^{1/s}.$$ By Assumption~\ref{ass:upper},
\[
\mu(B(x,r^-)) \le c_+(r^-)^s = \frac{k}{2n}.
\]
Therefore,
\[
\E[N_n(x,r^-)]
=
n\,\mu(B(x,r^-))
\le k/2,
\]
where,
\[
N_n(x,r)=\sum_{i=1}^n \mathbf 1\{\|X_{n,i}-x\|\le r\}.
\]

By Markov's inequality,
\[
\P\big(N_n(x,r^-)\ge k\big)
\le
\frac{\E[N_n(x,r^-)]}{k}
\le \frac12,
\]
and
\[
R_{n,k}(x)>r^-
\iff
N_n(x,r^-)<k,
\]
we have
\[
\P\big(R_{n,k}(x)>r^-\big)\ge \frac12.
\]

On the event \(\{R_{n,k}(x)>r^-\}\), we have \(R_{n,k}(x)^p\ge (r^-)^p\), and then
\[
\E[R_{n,k}(x)^p]
\ge
(r^-)^p\,\P\big(R_{n,k}(x)>r^-\big)
\ge
\frac12\,(r^-)^p.
\]
By the definition of \(r^-\), we obtain
\[
\E[R_{n,k}(x)^p]
\ge
\frac12\Big(\frac{k}{2c_+ n}\Big)^{p/s}.
\]
This completes the proof.
\end{proof}

\subsection{Proof of Theorem~\ref{thm:exact}}

\begin{proof}
The conclusion follows directly by combining the upper and lower bounds established in Theorems~\ref{thm:moment} and \ref{thm:lower}.
\end{proof}

\section{Experiments}

We provide additional details on the implementation of the experiments.

\subsection{Experiment 1 (Section~\ref{exp1:Moment convergence rate}) implementation details}
\label{app: Experiment 1 implementation details}

\paragraph*{Setup.}
Experiment~\ref{exp1:Moment convergence rate} uses interior query aggregation. For each MC replication, we sample with the global random seed 20260504
\[
X_{\mathrm{eval}}\sim \mathrm{Unif}([0.01,0.99]^d),\qquad |X_{\mathrm{eval}}|=1000,
\]
The number of NNs is chosen as $k=\lfloor n^\beta\rceil$ with $\beta$ on a 20-point grid in $[0.1,0.9]$ for slope fitting. We also impose a cap for $\frac{k}{n}\le 0.01$. 

\paragraph*{Data generating processes.}
The i.i.d. uniform data is used as the baseline in every panel. The dependent data is generated by the following processes, with three strength levels, as shown in Table~\ref{tab:exp1-dgp-settings}.

\begin{table}
  \caption{Dependence-strength settings for the three dependent DGPs in experiment 1. Weak settings are near i.i.d., while medium and strong settings increase dependence contrast.}
  \label{tab:exp1-dgp-settings}
  \centering
  \small
  \begin{tabular}{@{}lllll@{}}
    \toprule
    Model & Tuning parameter & Weak & Medium & Strong \\
    \midrule
    Linear Gaussian state-space (factor AR(1)) & AR coefficient $\rho$ & $0.02$ & $0.60$ & $0.98$ \\
    Hidden Markov model (2-state) & Stay probability $p_{\mathrm{stay}}$ & $0.70$ & $0.99$ & $0.9995$ \\
    Gaussian process (FFT, RBF) & Lengthscale $\ell$ & $1$ & $50$ & $200$ \\
    \bottomrule
  \end{tabular}
\end{table}

\begin{itemize}
\item \textbf{Linear Gaussian state-space (LSS).}
A shared latent AR(1) factor is generated as
\[
f_t=\rho f_{t-1}+\sqrt{1-\rho^2}\,\eta_t,\qquad \eta_t\sim\mathcal N(0,1).
\]
Then each coordinate is
\[
z_{t,j}=a f_t+\sqrt{1-a^2}\,\varepsilon_{t,j},\qquad \varepsilon_{t,j}\sim\mathcal N(0,1),
\]
with dimension-damped shared loading
\[
a=\frac{s\sqrt{|\rho|}}{d^\gamma},\qquad s=0.9,\ \gamma=0.3.
\]
(Implementation clips $a\in[0,0.999]$ and uses burn-in 400.)

\item \textbf{Hidden Markov model (HMM).}
Each coordinate $j$ has its own two-state chain $S_{t,j}\in\{0,1\}$ with
\[
\Pr(S_{t,j}=S_{t-1,j})=p_{\mathrm{stay}}.
\]
Conditional emissions are
\[
Z_{t,j}\mid S_{t,j}\sim \mathcal N \big((2S_{t,j}-1)\mu_{\mathrm{eff}},\,\sigma^2\big),
\]
where $\mu=1.75$, $\sigma=0.70$, and
\[
\mu_{\mathrm{eff}}=\mu\,|2p_{\mathrm{stay}}-1|^6.
\]
(This keeps weak settings close to i.i.d.\ after marginal transformation while separating medium/strong regimes; burn-in 400.)

\item \textbf{Gaussian process (FFT, RBF).}
For each coordinate, a stationary Gaussian sequence is generated with lag covariance
\[
k(h)=\exp \left(-\frac{h^2}{2\ell^2}\right),\qquad \ell\in\{1,50,200\},
\]
using circulant embedding and FFT spectral filtering ($O(n\log n)$), then truncating to length $n$.

\item \textbf{Transformation.}
All $d$ coordinates are transformed to have approximately $\mathrm{Unif}[0,1]$ marginal distributions while retaining dependence:
\[
U_{t,j}=\Phi(Z_{t,j}) \quad \text{(LSS/GP)},
\]
and, for HMM, via the stationary two-component mixture CDF
\[
F_{\mathrm{mix}}(z)=\tfrac12\Phi \left(\frac{z+\mu_{\mathrm{eff}}}{\sigma}\right)
+\tfrac12\Phi \left(\frac{z-\mu_{\mathrm{eff}}}{\sigma}\right).
\]
\end{itemize}

\subsection{Experiment 2 (Section~\ref{exp2: Entropy estimation}) implementation details}
\label{app: Experiment 2 implementation details}

\paragraph*{Data-generating process.}
For each intrinsic dimension $s$, we generate independent stationary Gaussian AR(1) latent processes across coordinates:
\[
Z_t=\rho Z_{t-1}+\sqrt{1-\rho^2}\,\varepsilon_t,
\qquad \varepsilon_t\sim \mathcal N(0,I_s).
\]
The temporal dependence levels are $\rho=0$ for the i.i.d.\ benchmark, $\rho=0.3$ for weak AR(1) dependence, and $\rho=0.6$ for strong AR(1) dependence. The latent coordinates are independent across dimensions, while each coordinate follows the same AR(1) dependence over time. The latent vector is embedded into a $5$-dimensional signal representation and then mapped into $20$ observed coordinates by a fixed orthonormal $20\times 5$ linear map, which induces a low-rank cross-coordinate dependence structure in the ambient observations. The PCA estimator fits a $5$-component PCA representation to the observed data and then estimates entropy using the top $s$ PCA coordinates, which, in this linear Gaussian setting, recover an orthogonal low-dimensional representation up to rotation.

\paragraph*{Estimators and simulation grid.}
All reported curves use the Kozachenko--Leonenko $k$-NN entropy estimator with Euclidean distance. The oracle intrinsic estimate is computed on the true latent $s$-dimensional sample, and the PCA estimate is computed on the top $s$ PCA coordinates. The current implementation uses $1000$ Monte Carlo replications with global random seed 20260504, burn-in $2000$, and neighbor schedule $k=\max\{2,\lceil n^{0.1}\rceil\}$.

\subsection{Verification of Assumptions~\ref{ass:geom}--\ref{ass:compact} for the synthetic experiments: proofs}
\label{app:assumption-verification}

This appendix verifies that the data-generating processes used in the
experiments satisfy the assumptions of the theory: we prove that each
Experiment-1 generator is geometrically $\alpha$-mixing with an explicit rate
and explicit constants (Assumption~\ref{ass:geom}; for the GP the bound is in
fact super-geometric), and that its marginal law satisfies the local
mass conditions with exponent $s=d$ and explicit constants
(Assumptions~\ref{ass:lower}--\ref{ass:upper}); Assumption~\ref{ass:compact}
is immediate because all samples take values in $[0,1]^d$. We also show
(Proposition~\ref{prop:moment-noncompact}) that the compact-support
assumption in Theorem~\ref{thm:moment} can be replaced by a marginal moment
condition, which covers the Gaussian design of Experiment~2.
Appendix~\ref{app:assumption-empirical} complements these proofs with a
empirical diagnostics on the simulation generators; these complement rather than replace the proofs.

Throughout, $\Phi$ and $\phi$ denote the standard normal distribution and
density functions, $\omega_d := \pi^{d/2}/\Gamma(d/2+1)$ the volume of the
$d$-dimensional unit ball, and $\lambda_2 := 2p_{\mathrm{stay}}-1$ the second
eigenvalue of the two-state symmetric transition matrix. We use the convention
$\log(1/0)=\infty$; when $\rho=0$ or $\lambda_2=0$, the corresponding
process is independent across time and Assumption~\ref{ass:geom} holds for
every finite $\gamma>0$. All three
Experiment-1 generators are \emph{stationary} sequences initialized from their
stationary laws (the implementation adds a burn-in, which is then redundant);
the triangular array of Section~\ref{sec:setup} is obtained by the sequential
embedding $X_{n,i}=X_i$, so the array coefficients $\alpha_m$ coincide with
the sequence coefficients.

Recall the three generators. Fix $d\ge 1$. In each case the observed vector is
$X_t\in[0,1]^d$.
\begin{itemize}
\item \textbf{Linear state space (LSS).} A stationary Gaussian AR(1) factor
$f_t=\rho f_{t-1}+\sqrt{1-\rho^2}\,\eta_t$ with $f_0\sim N(0,1)$,
$\eta_t\stackrel{\rm iid}{\sim}N(0,1)$, and observations
$X_{t,j}=\Phi\bigl(af_t+\sqrt{1-a^2}\,\varepsilon_{t,j}\bigr)$,
$j=1,\dots,d$, with $\varepsilon_{t,j}\stackrel{\rm iid}{\sim}N(0,1)$
independent of $[f_t]$ and loading $a=0.9\sqrt{|\rho|}/d^{0.3}\in[0,1)$.
\item \textbf{Hidden Markov model (HMM).} Per coordinate $j$ an independent
stationary two-state chain $S_{t,j}\in\{0,1\}$ with stay probability
$p_{\mathrm{stay}}$ and uniform initial law, emissions
$z_{t,j}=\mu_{\mathrm{eff}}(2S_{t,j}-1)+\sigma\xi_{t,j}$ with
$\xi_{t,j}\stackrel{\rm iid}{\sim}N(0,1)$, and
$X_{t,j}=F(z_{t,j})$, where
$F(z)=\tfrac12\Phi((z+\mu_{\mathrm{eff}})/\sigma)+\tfrac12\Phi((z-\mu_{\mathrm{eff}})/\sigma)$
is the exact stationary emission distribution function.
\item \textbf{Gaussian process (GP--RBF).} Per coordinate $j$ an independent
stationary Gaussian sequence $[w_{t,j}]_t$ with autocovariance
$r_h=e^{-h^2/(2\ell^2)}$ and $X_{t,j}=\Phi(w_{t,j})$.\footnote{The
implementation samples the finite window by circulant embedding, which
reproduces the covariance $r_h$ within the window exactly whenever the
embedding is nonnegative definite; negative eigenvalues, when they occur, are
at the level of numerical noise and are clipped at zero.}
\end{itemize}

\paragraph{Preliminaries}

We collect standard facts about the dependence coefficients
$\alpha(\mathcal A,\mathcal B)$, $\beta(\mathcal A,\mathcal B)$, and the
maximal correlation $\rho(\mathcal A,\mathcal B)$ between two $\sigma$-fields.

\begin{lemma}\label{lem:mixing-facts}
Let $[U_t]$, $[V_t]$ be random sequences on a common probability space.
\begin{enumerate}
\item[(i)] (Instantaneous functions.) If $Y_t=g(U_t)$ for a fixed measurable
$g$, then $\alpha_Y(m)\le\alpha_U(m)$ for all $m$, and similarly for $\beta$.
\item[(ii)] (Independent components.) If $[U_t]$ and $[V_t]$ are independent
of each other, then the paired sequence $W_t=(U_t,V_t)$ satisfies
$\alpha_W(m)\le\alpha_U(m)+\alpha_V(m)$,
$\beta_W(m)\le\beta_U(m)+\beta_V(m)$, and
$\rho_W(m)=\max\{\rho_U(m),\rho_V(m)\}$
\citep[Theorem~5.1(a)--(c)]{bradley2005basic}.
\item[(iii)] (Comparison.) For any two $\sigma$-fields,
$2\alpha(\mathcal A,\mathcal B)\le\beta(\mathcal A,\mathcal B)$ and
$4\alpha(\mathcal A,\mathcal B)\le\rho(\mathcal A,\mathcal B)$
\citep[eqs.~(1.11)--(1.12)]{bradley2005basic}.
\item[(iv)] (Gaussian sequences.) If $[U_t]$ is a stationary mean-zero
Gaussian sequence, its maximal correlation coefficient equals its maximal
\emph{linear} correlation coefficient:
$\rho_U(m)=\sup\,\lvert\mathrm{Corr}(L_1,L_2)\rvert$, the supremum running
over $L_1$ in the closed linear span of $[U_s:s\le t]$ and $L_2$ in that of
$[U_s:s\ge t+m]$ \citep[eq.~(7.1)]{bradley2005basic}, going back to
\citep{kolmogorov1960strong}.
\end{enumerate}
\end{lemma}

Part (i) is immediate from
$\sigma(Y_i:i\in I)\subseteq\sigma(U_i:i\in I)$. We also record how a
multiplicative constant is absorbed into the geometric rate, so that bounds of
the form $\alpha_m\le Ce^{-\gamma m}$ with $C\ge1$ imply
Assumption~\ref{ass:geom} verbatim.

\begin{remark}[Absorbing constants]\label{rem:absorb}
Suppose $\alpha_m\le Ce^{-\gamma m}$ for all $m\ge1$ with $C\ge1$, $\gamma>0$.
Since always $\alpha_m\le1/4$, setting $m_0:=\lceil\gamma^{-1}\log(4C)\rceil$
gives $\alpha_m\le\tfrac14 e^{-\gamma(m-m_0)_+}$, and hence
$\alpha_m\le e^{-\gamma' m}$ for all $m\ge1$ with
$\gamma':=\min\{\gamma,\ \log(4)/m_0\}>0$. Thus geometric decay with any
constant yields Assumption~\ref{ass:geom} after a benign adjustment of
$\gamma$.
\end{remark}

\paragraph{Assumption~\ref{ass:geom}: geometric \texorpdfstring{$\alpha$}{alpha}-mixing}

\begin{proposition}[LSS]\label{prop:mix-lss}
For any $d\ge1$, $\rho\in(-1,1)$, and $a\in[0,1)$, the LSS sequence satisfies
\[
\alpha_m\ \le\ \tfrac14\,|\rho|^m\qquad\text{for all }m\ge1,
\]
so Assumption~\ref{ass:geom} holds with $\gamma=\log(1/|\rho|)$.
\end{proposition}

\begin{proof}
Write $X_t=G(f_t,\varepsilon_t)$ with
$G(f,e)=(\Phi(af+\sqrt{1-a^2}\,e_j))_{j\le d}$ and
$\varepsilon_t=(\varepsilon_{t,1},\dots,\varepsilon_{t,d})$. By
Lemma~\ref{lem:mixing-facts}(i), $\alpha_X(m)\le\alpha_{(f,\varepsilon)}(m)$.
Since $[\varepsilon_t]$ is i.i.d.\ and independent of $[f_t]$,
Lemma~\ref{lem:mixing-facts}(ii) gives
$\alpha_{(f,\varepsilon)}(m)\le\alpha_f(m)+0$.

It remains to bound $\alpha_f(m)$ for the stationary Gaussian AR(1) factor,
for which $\mathrm{Cov}(f_t,f_{t+h})=\rho^{|h|}$. By
Lemma~\ref{lem:mixing-facts}(iv) it suffices to bound correlations of linear
combinations. Let $L_1$ lie in the closed linear span of $[f_s:s\le t]$ and
$L_2=\sum_j v_j f_{t+m+j}$ (finite sum, $j\ge0$). By the Markov property and
Gaussianity, $\E[f_{t+m+j}\mid f_s,\ s\le t]=\rho^{m+j}f_t$, so with
$b:=\sum_j v_j\rho^j$,
\[
\mathrm{Cov}(L_1,L_2)
=\mathrm{Cov}\bigl(L_1,\ \rho^m b\,f_t\bigr)
\le |\rho|^m\,|b|\,\sqrt{\Var(L_1)} .
\]
Moreover $\Var(L_2)\ge\Var\bigl(\E[L_2\mid f_{t+m}]\bigr)=b^2$, whence
$\lvert\mathrm{Corr}(L_1,L_2)\rvert\le|\rho|^m$. Taking suprema (and closing
in $L^2$) yields $\rho_f(m)=|\rho|^m$, and
Lemma~\ref{lem:mixing-facts}(iii) gives $\alpha_f(m)\le|\rho|^m/4$.
\end{proof}

For the three strength levels $\rho\in\{0.02,0.60,0.98\}$ this gives
$\gamma\approx 3.91$, $0.51$, $0.020$, i.e.\ dependence horizons
$1/\gamma\approx0.3$, $2$, $50$ steps.

\begin{proposition}[HMM]\label{prop:mix-hmm}
For any $d\ge1$ and $p_{\mathrm{stay}}\in(0,1)$, the HMM sequence satisfies
\[
\alpha_m\ \le\ \tfrac 14\,|\lambda_2|^m,\qquad \lambda_2=2p_{\mathrm{stay}}-1,
\]
so Assumption~\ref{ass:geom} holds with $\gamma=\log(1/|\lambda_2|)$.
\end{proposition}

\begin{proof}
As before, $X_t=G(S_t,\xi_t)$ with $\xi_t$ i.i.d.\ standard normal and
independent of the latent chains, so by
Lemma~\ref{lem:mixing-facts}(i)--(ii), $\alpha_X(m)\le\alpha_S(m)$, where
$S_t=(S_{t,1},\dots,S_{t,d})$ collects the $d$ independent stationary
two-state chains.

We first bound the maximal correlation of a single chain with transition
matrix
$T=\begin{psmallmatrix}p&1-p\\ 1-p&p\end{psmallmatrix}$,
$p=p_{\mathrm{stay}}$. The chain is reversible with respect to
$\pi=(\tfrac12,\tfrac12)$, so $T$ is self-adjoint on $L^2(\pi)$ with
eigenvalues $1$ and $\lambda_2=2p-1$; hence for every $h$ with $\pi h=0$,
\[
\lVert T^m h\rVert_{L^2(\pi)}\ \le\ |\lambda_2|^m\,\lVert h\rVert_{L^2(\pi)}.
\]
Let $f\in L^2(\sigma(S_u:u\le t))$ and $g\in L^2(\sigma(S_u:u\ge t+m))$ be
centered. By the Markov property, $\E[g\mid S_{t+m}]=:h(S_{t+m})$ with
$\pi h=0$ and $\Var(h(S_{t+m}))\le\Var(g)$, and
$\E[g\mid S_u,\ u\le t]=(T^m h)(S_t)$. Therefore, by Cauchy--Schwarz and
stationarity,
\[
\mathrm{Cov}(f,g)=\E\bigl[f\,(T^m h)(S_t)\bigr]
\le \sqrt{\Var(f)}\;\lVert T^m h\rVert_{L^2(\pi)}
\le |\lambda_2|^m\sqrt{\Var(f)\,\Var(g)},
\]
i.e.\ $\rho_{\rm chain}(m)\le|\lambda_2|^m$. By
Lemma~\ref{lem:mixing-facts}(ii), the $d$ independent chains satisfy
$\rho_S(m)=\rho_{\rm chain}(m)\le|\lambda_2|^m$, and
Lemma~\ref{lem:mixing-facts}(iii) gives
$\alpha_S(m)\le\tfrac14\rho_S(m)\le\tfrac14|\lambda_2|^m$.

(Alternatively, Davydov's identity
$\beta(m)=\E_\pi\lVert T^m(S_0,\cdot)-\pi\rVert_{\mathrm{TV}}$ for stationary
Markov chains \citep{davydov1968convergence}, together with
$T^m(s,\cdot)=\pi+\lambda_2^m(\delta_s-\pi)$ and the subadditivity in
Lemma~\ref{lem:mixing-facts}(ii), gives the slightly weaker bound
$\alpha_m\le\tfrac d4|\lambda_2|^m$.)
\end{proof}

For $p_{\mathrm{stay}}\in\{0.70,0.99,0.9995\}$: $\lambda_2\in\{0.4,0.98,0.999\}$
and $\log(1/\lambda_2)\approx0.92$, $0.020$, $0.0010$, i.e.\ horizons
$\approx1$, $50$, $1000$ steps. Note that the strong setting is
\emph{still geometrically mixing}; its horizon is simply comparable to the
sample sizes used in Experiment~1, which is the intended stress regime.

\begin{proposition}[GP--RBF]\label{prop:mix-gp}
Fix $\ell>0$ and $d\ge1$, and set
\[
f_{*,\ell}:=\sqrt{2\pi}\,\ell\,e^{-\pi^2\ell^2/2},
\qquad
K_\ell:=\frac{1}{f_{*,\ell}\bigl(1-e^{-2/\ell^2}\bigr)}.
\]
The GP--RBF sequence satisfies
\[
\alpha_m\ \le\ \tfrac14\,K_\ell\,e^{-m^2/(2\ell^2)}
\qquad\text{for all }m\ge1.
\]
In particular the decay is super-geometric, and Assumption~\ref{ass:geom}
holds for every fixed $\ell$ with
$\gamma(\ell)=\min\bigl\{m_\ell/\ell^2,\ \log(4)/m_\ell\bigr\}>0$, where
$m_\ell:=\bigl\lceil\ell\sqrt{2\log(4K_\ell)}\,\bigr\rceil$.
\end{proposition}

\begin{proof}
\emph{Step 1 (spectral floor).} Fix one coordinate and write $w_t$ for it.
Its spectral density is, by Poisson summation,
\[
f_\ell(\lambda)\ =\ \sum_{h\in\mathbb Z}e^{-h^2/(2\ell^2)}e^{-ih\lambda}
\ =\ \sqrt{2\pi}\,\ell\sum_{k\in\mathbb Z}
\exp\Bigl(-\tfrac{\ell^2(\lambda+2\pi k)^2}{2}\Bigr),
\qquad\lambda\in[-\pi,\pi],
\]
and all terms are positive, so keeping only $k=0$ gives
$f_\ell(\lambda)\ge f_{*,\ell}$ on $[-\pi,\pi]$. Hence, for any finitely
supported real $(a_u)_{u\ge0}$ and $U:=\sum_{u\ge0}a_u w_{t-u}$, writing
$A(\lambda):=\sum_{u\ge0}a_u e^{iu\lambda}$, Parseval's identity gives
\[
\Var(U)=\frac{1}{2\pi}\int_{-\pi}^{\pi}|A(\lambda)|^2 f_\ell(\lambda)\,d\lambda
\ \ge\ f_{*,\ell}\,\|a\|_2^2,
\]
and likewise $\Var(V)\ge f_{*,\ell}\,\|b\|_2^2$ for
$V:=\sum_{v\ge0}b_v w_{t+m+v}$.

\emph{Step 2 (Hankel bound on the cross-covariance).} With
$r_h=e^{-h^2/(2\ell^2)}$,
\[
\mathrm{Cov}(U,V)=\sum_{u,v\ge0}a_u b_v\,r_{m+u+v}
=\langle a,\ H_m b\rangle,
\qquad (H_m)_{uv}:=r_{m+u+v},
\]
so $|\mathrm{Cov}(U,V)|\le\lVert H_m\rVert_{\rm op}\|a\|_2\|b\|_2
\le\lVert H_m\rVert_{\rm HS}\|a\|_2\|b\|_2$. Counting the $q+1$ pairs with
$u+v=q$ and using $(m+q)^2\ge m^2+2mq$ and $m\ge1$,
\begin{align*}
\lVert H_m\rVert_{\rm HS}^2
&=\sum_{q\ge0}(q+1)\,e^{-(m+q)^2/\ell^2}
\le e^{-m^2/\ell^2}\sum_{q\ge0}(q+1)\,e^{-2mq/\ell^2} \\
&=\frac{e^{-m^2/\ell^2}}{\bigl(1-e^{-2m/\ell^2}\bigr)^{2}}
\le\frac{e^{-m^2/\ell^2}}{\bigl(1-e^{-2/\ell^2}\bigr)^{2}}.
\end{align*}
Combining the two steps,
$|\mathrm{Corr}(U,V)|\le K_\ell\,e^{-m^2/(2\ell^2)}$; the bound is uniform
over finite combinations and hence extends to the closed linear spans. By
Lemma~\ref{lem:mixing-facts}(iv),
$\rho_w(m)\le K_\ell\,e^{-m^2/(2\ell^2)}$.

\emph{Step 3 (coordinates and transform).} The $d$ coordinate sequences are
i.i.d.\ copies, so Lemma~\ref{lem:mixing-facts}(ii) gives
$\rho_W(m)=\rho_w(m)$ for the vector sequence $W_t=(w_{t,1},\dots,w_{t,d})$;
the coordinatewise transform $\Phi$ is instantaneous, so by
Lemma~\ref{lem:mixing-facts}(i) and (iii),
$\alpha_X(m)\le\alpha_W(m)\le\tfrac14\rho_W(m)
\le\tfrac14 K_\ell e^{-m^2/(2\ell^2)}$.

\emph{Step 4 (explicit geometric exponent).} By definition of $m_\ell$,
$K_\ell e^{-m_\ell^2/(2\ell^2)}\le\tfrac14$. For $m\ge m_\ell$, using
$m^2-m_\ell^2\ge 2m_\ell(m-m_\ell)$,
\[
\alpha_m\le\tfrac14 K_\ell e^{-m^2/(2\ell^2)}
\le\tfrac1{16}\,e^{-m_\ell(m-m_\ell)/\ell^2},
\]
while $\alpha_m\le\tfrac14$ always. Exactly as in Remark~\ref{rem:absorb},
these two bounds give $\alpha_m\le e^{-\gamma(\ell)m}$ for all $m\ge1$ with
$\gamma(\ell)=\min\{m_\ell/\ell^2,\log(4)/m_\ell\}$.
\end{proof}

\begin{remark}[Sharpness of the GP constants]\label{rem:gp-conservative}
For $\ell\ge1$, $\log K_\ell=\pi^2\ell^2/2+O(\log\ell)$, so
$m_\ell\approx\pi\ell^2$ and $\gamma(\ell)\asymp\ell^{-2}$: the provable
bound in Proposition~\ref{prop:mix-gp} activates only beyond
$m\approx\pi\ell^2$. This conservativeness comes from the worst-case spectral
floor $f_{*,\ell}$; at the level of pairs of observations the dependence is
far smaller, $\mathrm{Corr}(w_t,w_{t+m})=e^{-m^2/(2\ell^2)}$, i.e.\
negligible beyond $m\approx\ell$, which matches the measured integrated
autocorrelation times in Table~\ref{tab:assumption-diag}. The GP model is
included in Experiment~1 precisely as a stress test beyond the cleanly
parameterized geometric-mixing models.
\end{remark}

\paragraph{Assumptions~\ref{ass:lower}--\ref{ass:upper}: local mass, and Assumption~\ref{ass:compact}}

\begin{proposition}[Exactly uniform marginals]\label{prop:mass-product}
For the HMM and GP--RBF generators with any $d\ge1$ and any parameter value,
and for the LSS generator with $d=1$, the common marginal law is exactly
$\mu=\mathrm{Unif}([0,1]^d)$. Consequently, for every
$x\in(0,1)^d$, Assumptions~\ref{ass:lower} and~\ref{ass:upper} hold with
\[
s=d,\qquad c_-=c_+=\omega_d,\qquad
r_0=\mathrm{dist}(x,\partial[0,1]^d),
\]
and Assumption~\ref{ass:compact} holds with $D=\sqrt d$.
\end{proposition}

\begin{proof}
Each scalar sequence is stationary with a continuous, strictly increasing
marginal distribution function, and the implementation applies exactly that
function. For the LSS, $af_t+\sqrt{1-a^2}\,\varepsilon_{t,j}\sim N(0,1)$
because $f_t\sim N(0,1)$ at stationarity and
$a^2+(1-a^2)=1$; the transform is $\Phi$. For the HMM, $S_{t,j}$ is uniform at
stationarity, so $z_{t,j}$ has the two-component mixture distribution with
distribution function $F$ above, and the transform is $F$. For the GP,
$w_{t,j}\sim N(0,1)$ since $r_0=1$, and the transform is $\Phi$. In each case
the probability integral transform yields $X_{t,j}\sim\mathrm{Unif}(0,1)$
exactly.

For the HMM and the GP the $d$ coordinate sequences are independent by
construction (independent latent chains and noises; independent columns), so
$\mu$ is the product of uniform marginals, i.e.\ Lebesgue measure on
$[0,1]^d$; for $d=1$ this holds trivially for the LSS as well. Finally, for
$x\in(0,1)^d$ and $r\le\mathrm{dist}(x,\partial[0,1]^d)$ we have
$B(x,r)\subseteq[0,1]^d$ and hence
$\mu(B(x,r))=\mathrm{vol}(B(x,r))=\omega_d r^d$.
\end{proof}

\begin{proposition}[LSS with $d\ge2$: Gaussian copula]\label{prop:mass-lss}
For the LSS generator with $d\ge2$, the marginal law $\mu$ is the Gaussian
copula measure with equicorrelation matrix
$\Sigma=(1-a^2)I_d+a^2\mathbf 1\mathbf 1^\top$, whose density
\[
c_\Sigma(u)\ =\
\frac{\phi_\Sigma\bigl(\Phi^{-1}(u_1),\dots,\Phi^{-1}(u_d)\bigr)}
{\prod_{j=1}^d\phi\bigl(\Phi^{-1}(u_j)\bigr)},
\qquad u\in(0,1)^d,
\]
is continuous and strictly positive on the open cube (for $a^2<1$, $\Sigma$ is
nonsingular with eigenvalues $1-a^2$ and $1+(d-1)a^2$). Consequently, for
every $x\in(0,1)^d$, Assumptions~\ref{ass:lower} and~\ref{ass:upper} hold with
$s=d$: setting $r_0=r_0(x)=\tfrac12\,\mathrm{dist}(x,\partial[0,1]^d)$ and
$K_x=\overline{B(x,r_0)}\subset(0,1)^d$,
\[
c_-=\omega_d\min_{K_x}c_\Sigma>0,
\qquad
c_+=\omega_d\max_{K_x}c_\Sigma<\infty,
\]
and for all $r\le r_0$,
$c_-r^d\le\mu(B(x,r))=\int_{B(x,r)}c_\Sigma(u)\,du\le c_+r^d$.
\end{proposition}

\begin{proof}
The joint law of $(z_{t,1},\dots,z_{t,d})$ is $N_d(0,\Sigma)$ with
$\mathrm{Corr}(z_{t,j},z_{t,k})=a^2$ for $j\ne k$, so
$X_t=(\Phi(z_{t,1}),\dots,\Phi(z_{t,d}))$ has the stated copula law; the
density formula and its continuity and positivity on $(0,1)^d$ are standard
for nonsingular $\Sigma$. The compact set $K_x$ lies in the open cube, so
$c_\Sigma$ attains a positive minimum and finite maximum on it, and for
$r\le r_0$ the ball $B(x,r)$ is contained in $K_x\subseteq[0,1]^d$; the
displayed bounds follow by integrating the density over $B(x,r)$.
\end{proof}

Propositions~\ref{prop:mass-product} and~\ref{prop:mass-lss} make the
division of labor in the theory concrete: the temporal dependence deforms the
\emph{joint} law across time (Assumption~\ref{ass:geom}) but the \emph{local
mass exponent} of the marginal remains $s=d$ at every dependence level; for
the LSS with $d\ge2$ only the constants $c_\pm$ become query-dependent. Since
the slope diagnostic of Experiment~1 estimates the exponent, not the
constants, the predicted slope $p/d$ is unaffected. This is confirmed
empirically in Appendix~\ref{app:assumption-empirical}
(Table~\ref{tab:assumption-diag}): the fitted local-mass exponents satisfy
$\hat s\in[1.000,1.001]$ (target $d=1$) and $\hat s\in[2.976,3.007]$ (target
$d=3$) across all models and strengths.

\begin{remark}[Experiment 2]\label{rem:exp2}
The Experiment-2 process is the stationary Gaussian AR(1) latent vector
$Z_t=\rho Z_{t-1}+\sqrt{1-\rho^2}\,\varepsilon_t\in\R^s$ with independent
coordinates, observed as $Y_t=Q\,(Z_t^\top,0,\dots,0)^\top\in\R^{20}$ for a
fixed orthonormal $Q\in\R^{20\times5}$ ($Q^\top Q=I_5$).
(i)~\emph{Mixing:} each coordinate of $Z_t$ is the AR(1) factor of
Proposition~\ref{prop:mix-lss}, whose maximal correlation is $\rho^m$; since
the coordinates are independent, Lemma~\ref{lem:mixing-facts}(ii) gives
$\rho_Z(m)=\rho^m$, hence $\alpha_m\le\tfrac14\rho^m$ by
Lemma~\ref{lem:mixing-facts}(iii), and Assumption~\ref{ass:geom} holds with
$\gamma=\log(1/\rho)$ for $\rho\in\{0.3,0.6\}$;
$Y_t$ is an instantaneous function of $Z_t$, so it inherits the bound.
(ii)~\emph{Local mass with intrinsic exponent:} $Q$ is an isometry onto its
range, so $\lVert Y_t-y\rVert=\lVert Z_t-z\rVert$ for query points
$y=Q(z^\top,0,\dots,0)^\top$ on the support, and hence
$\mu_Y(B(y,r))=\P(\lVert Z-z\rVert\le r)
=\int_{B_s(z,r)}\phi_s(v)\,dv$, where $\phi_s$ is the $N(0,I_s)$ density.
For $r\le1$ this is bounded between
$\omega_s r^s\min_{B_s(z,1)}\phi_s$ and $\omega_s r^s\max_{B_s(z,1)}\phi_s$,
so Assumptions~\ref{ass:lower}--\ref{ass:upper} hold with exponent $s$ --- the
\emph{intrinsic} dimension --- even though the ambient dimension is $20$.
This is exactly the mechanism that the PCA entropy experiment probes.
(iii)~\emph{Compact support:} Assumption~\ref{ass:compact} fails as stated
(Gaussian support is unbounded), but the Gaussian marginal has finite moments
of every order, so Proposition~\ref{prop:moment-noncompact} below applies
with any $q>p$: the moment upper bound of Theorem~\ref{thm:moment} holds
verbatim for the Experiment-2 design, with no additional logarithmic factors.
Claims (i) and (ii) are illustrated by empirical diagnostics on the Experiment-2
generator in Appendix~\ref{app:assumption-empirical}
(Figure~\ref{fig:diag-exp2}).
\end{remark}

\paragraph{Relaxing the compact-support assumption}

Assumption~\ref{ass:compact} enters the proof of Theorem~\ref{thm:moment}
only through the deterministic bound $R_{n,k}(x)\le D$ on the far-tail event.
The following proposition replaces it by a marginal moment condition.

\begin{proposition}[Moments without compact support]\label{prop:moment-noncompact}
Assume Assumptions~\ref{ass:geom} and~\ref{ass:lower} at $x$, and suppose
that $M_q:=\E\lVert X_{n,1}-x\rVert^q<\infty$ for some $q>p$, where $p>0$ is
the moment order of interest. Then there exist constants $K_0'$ and $C'$,
depending only on $(p,q,s,\gamma,c_-,r_0,M_q)$, such that
\[
\E\bigl[R_{n,k}(x)^p\bigr]\ \le\ C'\Bigl(\frac kn\Bigr)^{p/s}
\qquad\text{whenever }K_0'\log n\le k\le\kappa n .
\]
\end{proposition}

\begin{proof}
Fix
\[
A\ \ge\ A_0:=\frac{p/s+p/q}{\,1-p/q\,}\qquad(\text{note }A_0\ge p/s).
\]

\emph{Step 1 (tail bound with remainder $n^{-A}$).} Apply Theorem~\ref{thm:tail} with the chosen $A$ and $K_0'\ge K_A$. With $r_*=(8k/(c_-n))^{1/s}$ and $J_0$ maximal such that $2^{J_0}r_*\le r_0$, it gives
\[
\P\bigl(R_{n,k}(x)>2^j r_*\bigr)
\le 4\exp\Bigl(-c_A\,2^{js}\,\frac{k}{\log n}\Bigr)+C_A n^{-A},
\qquad 0\le j\le J_0.
\]
Compact support is not required for this tail bound.

\emph{Step 2 (local dyadic part).} Exactly as in the proof of
Theorem~\ref{thm:moment},
\[
\E\bigl[R_{n,k}(x)^p\,\mathbf 1\{R_{n,k}(x)\le 2^{J_0}r_*\}\bigr]
\le r_*^p+\sum_{j=0}^{J_0-1}(2^{j+1}r_*)^p\,\P\bigl(R_{n,k}(x)>2^jr_*\bigr)
\le C r_*^p + C n^{-A},
\]
using $\sum_{j\ge0}2^{jp}e^{-c_A 2^{js}K_0'}<\infty$ for the exponential part
and $r_*^p\sum_{j<J_0}2^{jp}\le C(r_0,p)$ (from $2^{J_0}\asymp r_0/r_*$) for
the remainder part. Since $A\ge p/s$ and $k\ge1$,
$n^{-A}\le n^{-p/s}\le(k/n)^{p/s}$.

\emph{Step 3 (far tail via H\"older).} Let
$E_J:=\{R_{n,k}(x)>2^{J_0}r_*\}$ and
$M_n:=\max_{1\le i\le n}\lVert X_{n,i}-x\rVert$, so that
$R_{n,k}(x)\le M_n$ almost surely and
$\E[M_n^q]\le\sum_{i=1}^n\E\lVert X_{n,i}-x\rVert^q=nM_q$ by the common
marginal. At $j=J_0$ one has $2^{J_0 s}k\asymp c_-r_0^s\,n/8$, so the
exponential term in Step 1 is $e^{-cn/\log n}\le n^{-A}$ for all large $n$,
whence $\P(E_J)\le C n^{-A}$. H\"older's inequality with exponents
$\bigl(q/p,\ (1-p/q)^{-1}\bigr)$ then gives
\begin{align*}
\E\bigl[R_{n,k}(x)^p\,\mathbf 1_{E_J}\bigr]
&\le\E[M_n^q]^{p/q}\,\P(E_J)^{1-p/q}
\le C\,(nM_q)^{p/q}\,n^{-A(1-p/q)} \\
&= C'\,n^{p/q-A(1-p/q)}
\le C'\,n^{-p/s},
\end{align*}
by the choice $A\ge A_0$, and $n^{-p/s}\le(k/n)^{p/s}$ since $k\ge1$.
Combining Steps 2--3 completes the proof.
\end{proof}

\begin{remark}[Choice of the blocking exponent]\label{rem:blocking}
Theorem~\ref{thm:tail} allows an arbitrary remainder $C_A n^{-A}$. For compact support, $A>p/s$ suffices; with only a finite $q$th marginal moment, the proof above uses $A\ge(p/s+p/q)/(1-p/q)$. The threshold on $k/\log n$ and the constants may therefore depend on the moment order, while the rate remains $(k/n)^{p/s}$.
\end{remark}

\subsection{Empirical diagnostics for the synthetic assumptions}
\label{app:assumption-empirical}

The preceding appendix verifies the assumptions analytically. Here we provide
complementary empirical diagnostics using the same data-generating mechanisms,
dependence-parameter settings, and global random seed (20260504) as
Experiments~1 and~2. For Experiment~2, we also use the same fixed embedding
matrix $Q$. The diagnostic sample sizes, replication counts, and evaluation
grids are chosen specifically for these checks and are stated below.
These diagnostics do not estimate the full past--future mixing coefficients.
Rather, they check observable implications of the proved models, including
autocorrelation decay, pairwise event-level dependence, effective sample size,
and local ball-mass scaling. Wherever an exact
finite-dimensional prediction is available, the corresponding overlay uses
that analytic prediction with \emph{no fitted parameters}.
Table~\ref{tab:assumption-diag} collects the headline numbers.

\paragraph{Assumption~\ref{ass:geom}, correlation view: decay at the proved rates}

For each Experiment-1 generator and dependence level ($d=1$, $n=6400$, $30$
replications) we compute the mean sample autocorrelation of the series and
overlay the analytic prediction. For the two Gaussian-copula models the
overlay is exact: if the latent Gaussian pair at lag $m$ has correlation
$r_m$, then
$\mathrm{Corr}\bigl(\Phi(w_t),\Phi(w_{t+m})\bigr)
=\tfrac{6}{\pi}\arcsin(r_m/2)$, with $r_m=a^2\rho^m$ for the LSS and
$r_m=e^{-m^2/(2\ell^2)}$ for the GP; for the HMM the guide is the proved
latent rate $\lambda_2^m$ of Proposition~\ref{prop:mix-hmm}, anchored at lag
one. Figure~\ref{fig:diag-acf} shows the result: the measured autocorrelations
track the corresponding analytic or model-implied decay curves until they
reach the Monte Carlo noise floor. This is consistent with geometric decay for
the LSS and HMM and with faster-than-geometric pairwise correlation decay for
the GP. The strong settings decay slowly, with characteristic scales of
$50$--$1000$ steps --- the intended stress regime.

\begin{figure}[t]
\centering
\includegraphics[width=\textwidth]{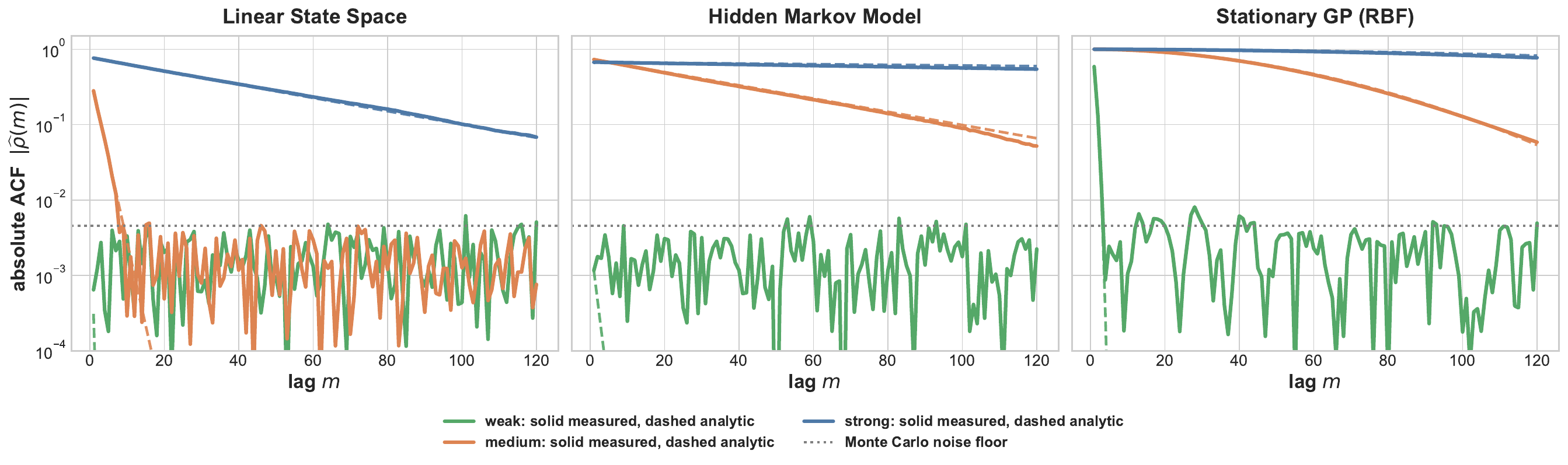}
\caption{Assumption~\ref{ass:geom}, correlation view. Mean absolute
autocorrelation of the generated series ($d=1$, $n=6400$, $30$ repetitions;
solid) against the analytic decay rates (dashed; exact
$\frac6\pi\arcsin(r_m/2)$ overlays for the Gaussian-copula models, geometric
guide at the proved rate $\lambda_2^m$ for the HMM). Measured autocorrelations track the corresponding analytic or model-implied
decay curves until the Monte Carlo noise floor (dotted).}
\label{fig:diag-acf}
\end{figure}

\paragraph{Assumption~\ref{ass:geom}, event view: a total-variation surrogate}

Autocorrelation measures only linear dependence, whereas
Assumption~\ref{ass:geom} constrains dependence between \emph{events}. As a
distribution-level check we estimate the dependence between $X_t$ and
$X_{t+m}$ in total variation,
\[
\widehat{\mathrm{TV}}_{\mathrm{pair}}(m)\ =\ \tfrac12\sum_{i,j=1}^{B}
\bigl|\hat P_{ij}(m)-\hat p_i\hat q_j\bigr|,
\]
where $\hat P(m)$ is the $B\times B$ histogram of the pairs $(X_t,X_{t+m})$
($B=10$ equiprobable bins, since the marginals are exactly uniform by
Proposition~\ref{prop:mass-product}), in the spirit of histogram estimators for pairwise dependence such as
\citep{mcdonald2015estimating}. This statistic measures the total-variation
distance between the binned pair law of $(X_t,X_{t+m})$ and the corresponding
product law. It is a pairwise event-level diagnostic, not an estimator of the
full past--future $\beta$- or $\alpha$-mixing coefficient. The i.i.d.\ generator provides
the estimation floor of this estimator at this sample size: a curve that
reaches the floor is indistinguishable from independence.
Figure~\ref{fig:diag-beta} shows that the pairwise event-level dependence
decays toward the i.i.d.\ estimation floor, with the same qualitative ordering
of weak, medium, and strong dependence as in Figure~\ref{fig:diag-acf}. This
provides a nonlinear diagnostic complementary to autocorrelation.

\begin{figure}[t]
\centering
\includegraphics[width=\textwidth]{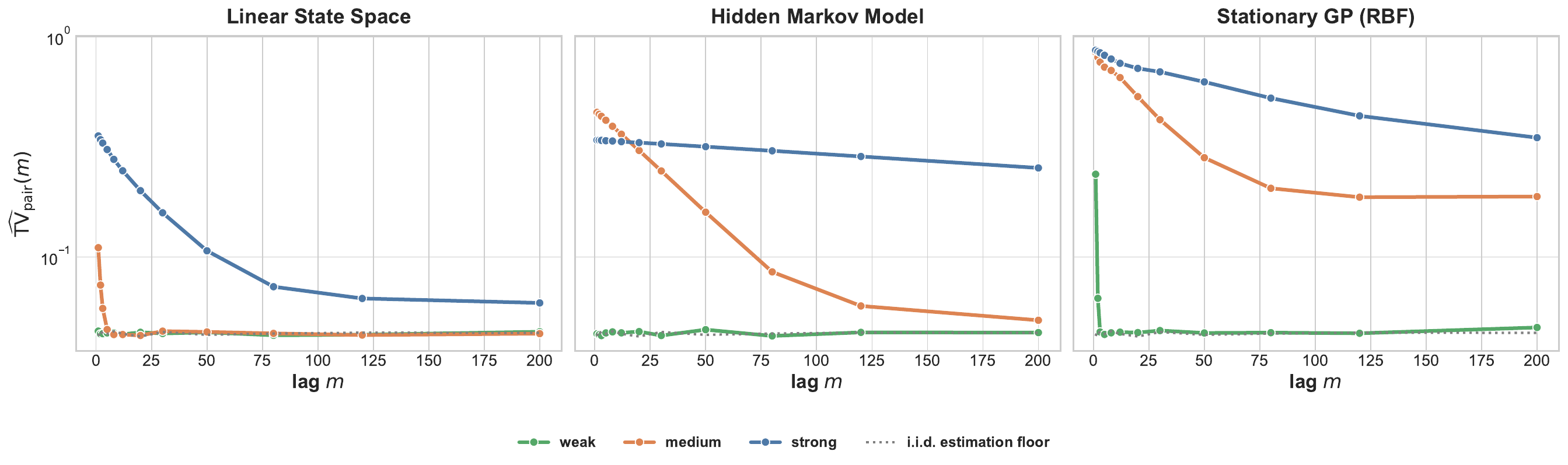}
\caption{Assumption~\ref{ass:geom}, event view. Pairwise total-variation
dependence $\widehat{\mathrm{TV}}_{\mathrm{pair}}(m)$ between $X_t$ and $X_{t+m}$ (histogram estimator,
$10\times10$ equiprobable bins; $d=1$, $n=6400$, $30$ repetitions). The gray
dotted curve is the i.i.d.\ estimation floor; reaching it is
indistinguishable from independence. Pairwise event-level dependence decays toward the i.i.d.\ estimation floor
with the expected ordering across dependence strengths.}
\label{fig:diag-beta}
\end{figure}

\paragraph{From Assumption~\ref{ass:geom} to Figure~\ref{fig:exp1-slope-grid}: effective sample size}

The theory predicts that dependence enters the radius bounds only through
$\gamma$ in the constants and through the $k\gtrsim\log n$ regime --- the
exponent of $(k/n)^{1/s}$ is dependence-free. The natural finite-sample
summary of Assumption~\ref{ass:geom} is the integrated autocorrelation time
$\hat\tau=1+2\sum_{m\ge1}\widehat{\mathrm{acf}}(m)$ (Sokal windowing) and
the effective sample size $n_{\mathrm{eff}}=n/\hat\tau$.
Figure~\ref{fig:diag-neff} plots the relative error of the fitted
Figure~\ref{fig:exp1-slope-grid} slopes (mean over $p$, $d=1$) against
$n_{\mathrm{eff}}$ at $n=6400$. The pattern is sharp: every configuration
with $n_{\mathrm{eff}}\gtrsim2500$ reproduces the predicted slope $p/d$ with
the same $\approx3\%$ error as the i.i.d.\ baseline, every deviating
configuration has $n_{\mathrm{eff}}\le90$ (provable dependence horizons of
$50$--$1000$ steps against $n=6400$), and the error generally increases as
$n_{\mathrm{eff}}$ falls. The pattern is consistent with the interpretation that the deviations under
strong dependence are finite \emph{effective}-sample effects rather than a
change in the asymptotic exponent.

\begin{figure}[t]
\centering
\includegraphics[width=0.62\linewidth]{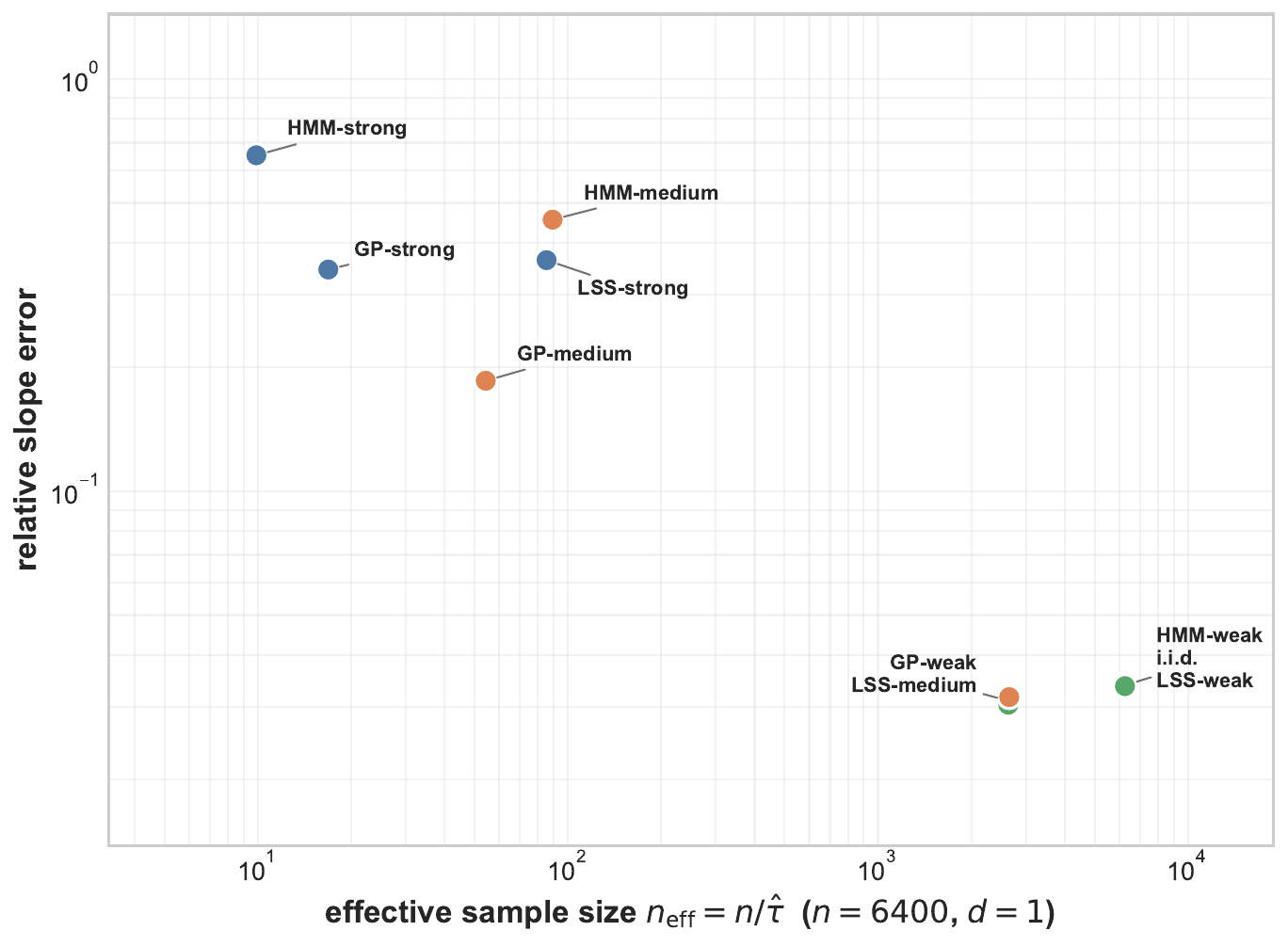}
\caption{Relative error of the fitted slopes in
Figure~\ref{fig:exp1-slope-grid} (mean over $p$, $d=1$, $n=6400$) against
the measured effective sample size $n_{\mathrm{eff}}=n/\hat\tau$.
Configurations with $n_{\mathrm{eff}}\gtrsim2500$ match the i.i.d.\
accuracy; all deviating configurations have $n_{\mathrm{eff}}\le90$.}
\label{fig:diag-neff}
\end{figure}

\paragraph{Assumptions~\ref{ass:lower}--\ref{ass:upper}: local-mass exponent and constant}

We estimate $\mu(B(x,r))$ by the empirical fraction of sample points in
$B(x,r)$, averaged over $200$ interior query points
$x\in[0.2,0.8]^d$ and over repetitions ($d=1$: $n=20{,}000$, $20$
repetitions; $d=3$: $n=40{,}000$, $8$ repetitions), on a logarithmic radius
grid $r\in[0.02,0.15]$ chosen so that all balls stay inside $[0,1]^d$. Two
readouts follow. First, the fitted log--log slope $\hat s$ of
$\hat\mu(B(x,r))$ against $r$ estimates the local-mass exponent: across
\emph{every} model and dependence level,
$\hat s\in[1.000,1.001]$ for target $d=1$ and $\hat s\in[2.976,3.007]$ for
target $d=3$ (Table~\ref{tab:assumption-diag}). Second, the ratio
$\hat\mu(B(x,r))/(\omega_d r^d)$ in Figure~\ref{fig:diag-mass} is approximately flat in $r$, as expected under the local scaling
$\mu(B(x,r))\asymp r^d$, at level $1$ for
the product-uniform models (Proposition~\ref{prop:mass-product}) and at the
query-averaged copula density, also $\approx1$, for the LSS with $d=3$
(Proposition~\ref{prop:mass-lss}). This is consistent with the division of labor
behind the experiments: dependence deforms the joint law across time
(Assumption~\ref{ass:geom}) but leaves the marginal local geometry, hence
the exponent $s=d$ probed by the slope diagnostic, untouched at every
dependence level.

\begin{figure}[t]
\centering
\includegraphics[width=\textwidth]{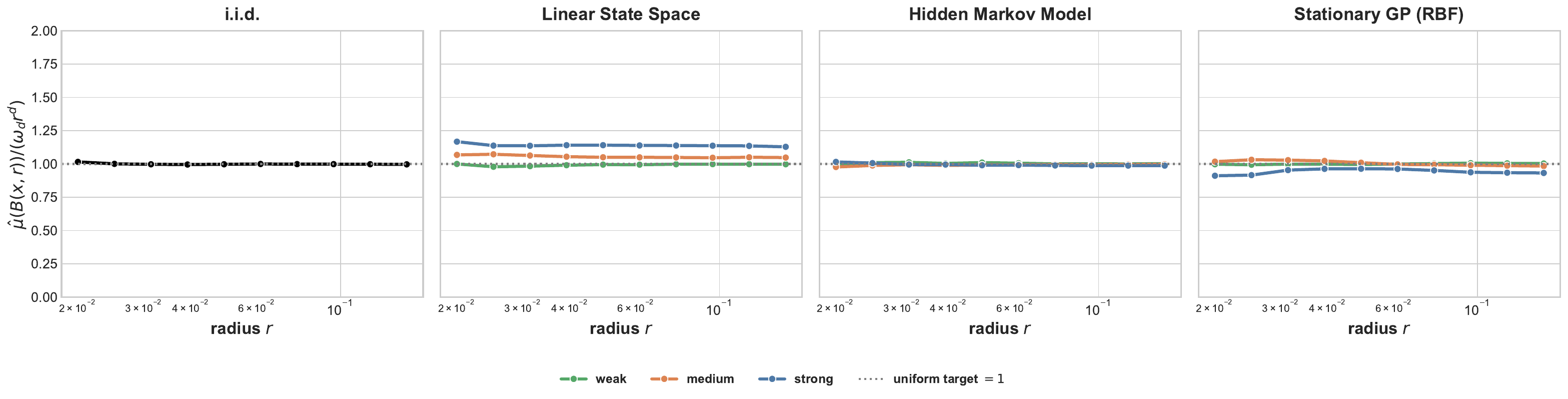}
\caption{Assumptions~\ref{ass:lower}--\ref{ass:upper}. Ratio
$\hat\mu(B(x,r))/(\omega_d r^d)$ for $d=3$, averaged over $200$ interior
queries ($n=40{,}000$ per repetition). Approximate flatness over the displayed radius range is consistent with
$\mu(B(x,r))\asymp r^d$; the level identifies the local constant ($1$ for
the exactly uniform models; the query-averaged copula density for the LSS).
The fitted exponents are $\hat s\in[2.976,3.007]$
(Table~\ref{tab:assumption-diag}).}
\label{fig:diag-mass}
\end{figure}

\paragraph{Experiment 2: mixing and the intrinsic exponent}

The Experiment-2 diagnostics use the exact generator of
Remark~\ref{rem:exp2} (same embedding $Q$ and seed as the experiment;
$s\in\{3,5\}$, $\rho\in\{0.3,0.6\}$). For Assumption~\ref{ass:geom}, every
observed coordinate of $Y_t\in\R^{20}$ is a fixed linear function of the
latent AR(1), so its autocorrelation is exactly $\rho^m$: the left panel of
Figure~\ref{fig:diag-exp2} shows the measured curves on this parameter-free
overlay down to the Monte Carlo noise floor, and the measured
$\hat\tau$ is within $2\%$ of the analytic $(1+\rho)/(1-\rho)$ ($1.83$ vs.\
$1.86$ at $\rho=0.3$; $4.07$ vs.\ $4.00$ at $\rho=0.6$), giving
$n_{\mathrm{eff}}\approx3500$ and $1570$ at $n=6400$ --- consistent with moderate finite-sample dependence in Experiment~2. For
Assumptions~\ref{ass:lower}--\ref{ass:upper}, we estimate the ball mass of
the embedded cloud in $\R^{20}$ at on-support queries
$y=Q(z^\top,0,\dots,0)^\top$ with $z\sim N(0,I_s)$ ($n=40{,}000$, $6$
repetitions, $200$ queries, $r\in[0.45,1.1]$). Averaged over independent on-support queries, $Z-z\sim N(0,2I_s)$, so the
exact mass is $F_{\chi^2_s}(r^2/2)$ --- an $s$-dimensional quantity, free of
$\rho$ and of any fitted parameter. The right panel of
Figure~\ref{fig:diag-exp2} shows the measured mass lying \emph{on} this
curve: the fitted exponents are $\hat s\in[2.83,2.85]$ (target $s=3$) and
$\hat s\in[4.81,4.85]$ (target $s=5$), against exact slopes of the analytic
curve over the same radii of $2.84$ and $4.81$. (The exact slope sits
slightly below $s$ at finite radii because the Gaussian density varies
within the ball; it tends to $s$ as $r\to0$, and the measured values match
it to two decimals.) The local-mass exponent of the $20$-dimensional data is
therefore the \emph{intrinsic} dimension $s$ --- an ambient exponent of $20$
would fall by roughly twelve decades across this radius grid --- which is
precisely the mechanism that the PCA entropy experiment exploits.
Assumption~\ref{ass:compact} fails for the Gaussian design but is not
needed: Proposition~\ref{prop:moment-noncompact} applies with any $q>p$.

\begin{figure}[t]
\centering
\includegraphics[width=\textwidth]{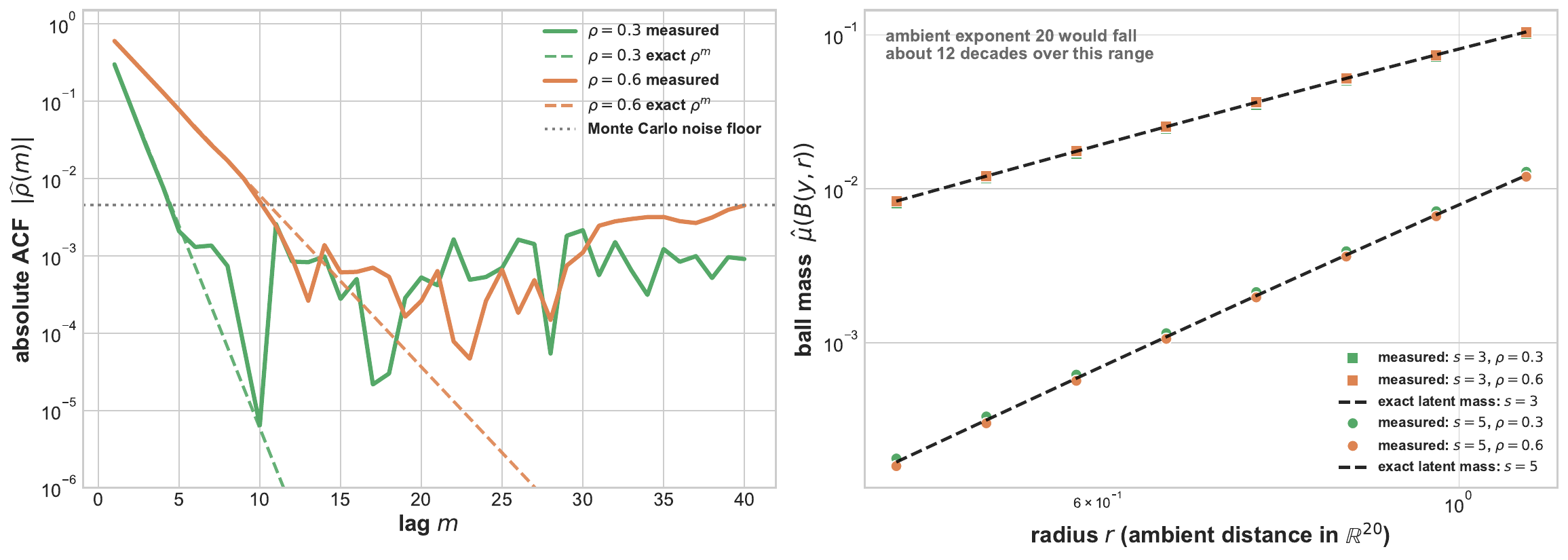}
\caption{Experiment-2 diagnostics on the exact generator. Left: mean
absolute autocorrelation of the observed coordinates in $\R^{20}$ against
the exact $\rho^m$ (dashed; no fitted parameters). Right: empirical ball
mass of the embedded cloud at on-support queries against the exact
$s$-dimensional latent mass $F_{\chi^2_s}(r^2/2)$ (dashed): the measured
points lie on the $s$-dimensional curves, with fitted exponents
$\hat s\approx2.84$ and $4.81$--$4.85$ for targets $s=3$ and $s=5$ --- the
intrinsic exponent, not the ambient dimension $20$.}
\label{fig:diag-exp2}
\end{figure}

\paragraph{Summary}

Table~\ref{tab:assumption-diag} summarizes the analytic rates alongside the
empirical diagnostics computed on the exact implementation ($d=1$,
$n=6400$, $40$ repetitions): the integrated autocorrelation time
$\hat\tau$, the effective sample size $n_{\mathrm{eff}}=n/\hat\tau$, and the
fitted local-mass exponents $\hat s$ ($200$ interior queries). The measured
$\hat\tau$ matches the analytic dependence horizons, the configurations
whose Figure~\ref{fig:exp1-slope-grid} slopes deviate from $p/d$ are exactly
those with $n_{\mathrm{eff}}\lesssim10^2$ at $n=6400$, and
$\hat s\approx d$ throughout.

\begin{table}[t]
\centering
\small
\begin{tabular}{llcccccc}
\toprule
Model & Strength & Rate & $\gamma$ & $\hat\tau$ & $n_{\mathrm{eff}}$
& $\hat s$ ($d{=}1$) & $\hat s$ ($d{=}3$) \\
\midrule
i.i.d. & --- & --- & $\infty$ & $1.02$ & $6302$ & $1.000$ & $2.995$ \\
\addlinespace[2pt]
LSS & weak & $\rho=0.02$ & $3.91$ & $1.02$ & $6265$ & $1.000$ & $3.006$ \\
LSS & medium & $\rho=0.60$ & $0.51$ & $2.41$ & $2652$ & $1.001$ & $2.989$ \\
LSS & strong & $\rho=0.98$ & $0.020$ & $75.1$ & $85$ & $1.000$ & $2.991$ \\
\addlinespace[2pt]
HMM & weak & $\lambda_2=0.40$ & $0.92$ & $1.02$ & $6307$ & $1.000$ & $2.999$ \\
HMM & medium & $\lambda_2=0.98$ & $0.020$ & $71.8$ & $89$ & $1.000$ & $3.007$ \\
HMM & strong & $\lambda_2=0.999$ & $0.0010$ & $648$ & $9.9$ & $1.000$ & $2.988$ \\
\addlinespace[2pt]
GP--RBF & weak & $\ell=1$ & --- & $2.43$ & $2633$ & $1.000$ & $3.005$ \\
GP--RBF & medium & $\ell=50$ & --- & $118$ & $54$ & $1.000$ & $2.976$ \\
GP--RBF & strong & $\ell=200$ & --- & $380$ & $17$ & $1.001$ & $3.006$ \\
\bottomrule
\end{tabular}
\caption{Assumption diagnostics for the Experiment-1 generators. ``Rate'' is
the analytic per-step dependence rate (Propositions~\ref{prop:mix-lss}
and~\ref{prop:mix-hmm}: $\alpha_m\le\frac14\rho^m$ and
$\alpha_m\le\frac14\lambda_2^m$; for the GP, Proposition~\ref{prop:mix-gp}
gives the super-geometric bound $\alpha_m\le\frac14 K_\ell
e^{-m^2/(2\ell^2)}$, whose conservative constants are discussed in
Remark~\ref{rem:gp-conservative}, and the pairwise dependence horizon is
$\ell$), $\gamma$ the implied exponent in Assumption~\ref{ass:geom},
$\hat\tau$ and $n_{\mathrm{eff}}=n/\hat\tau$ the measured integrated
autocorrelation time and effective sample size at $n=6400$ ($d=1$), and
$\hat s$ the fitted local-mass exponent of
Assumptions~\ref{ass:lower}--\ref{ass:upper} (targets $d=1$ and $d=3$).}
\label{tab:assumption-diag}
\end{table}

\subsection{Experiment 3 (Section~\ref{exp3: Real-data_evaluation}) implementation details}
\label{app:real-data-implementation}

We follow the real-data benchmark setup of \cite{goswami2024moment} for long-horizon forecasting, short-horizon forecasting, and classification, using the same task definitions, evaluation metrics, and preprocessing procedures. In addition to the source-to-target short-horizon setting considered in \cite{goswami2024moment}, we also report a within-group short-horizon evaluation. The setup for each task is as follows.

\begin{itemize}
    \item \textbf{Long-horizon forecasting.} We report 8 datasets from the Informer long-horizon forecasting benchmark: ETTm1, ETTm2, ETTh1, ETTh2, Electricity, Traffic, Weather, and ILI. The evaluation metrics are MSE and mean absolute error (MAE). We use a look-back window length of $512$ and a horizontal 60/10/30 train/validation/test split, with forecast horizons $\{24,36,48,60\}$ for ILI and $\{96,192,336,720\}$ otherwise. Hyperparameters are selected by 5-fold time-series cross-validation when the training set contains at least 20 windows, and by validation-holdout tuning otherwise.

    \item \textbf{Short-horizon forecasting.} We use 6 benchmark groups from the Monash time series forecasting archive: M4-Yearly, M4-Quarterly, M4-Monthly, M3-Yearly, M3-Quarterly, and M3-Monthly. The evaluation metric is sMAPE, and the input length is $512$. In the source-to-target transfer setting, hyperparameters are selected on a source group and evaluated directly on a different target group with the same frequency. We also report a within-group short-horizon evaluation, which is not included in the original MOMENT benchmark. In this setting, each dataset group is randomly split into a 70\% source portion and a 30\% target portion; hyperparameters are selected by 5-fold cross-validation on the source portion, and evaluation is performed on the target portion from the same group.

    \item \textbf{Classification.} We use 91 datasets from the UCR classification archive. The evaluation metric is accuracy. Inputs are represented as length-$512$ time series, and we use the official train/test splits provided by the UCR archive. Hyperparameters are selected using only the official training split, via stratified cross-validation when feasible, and final performance is evaluated on the official test split.
\end{itemize}

\begin{figure}[htbp]
    \centering
    \includegraphics[width=1\linewidth]{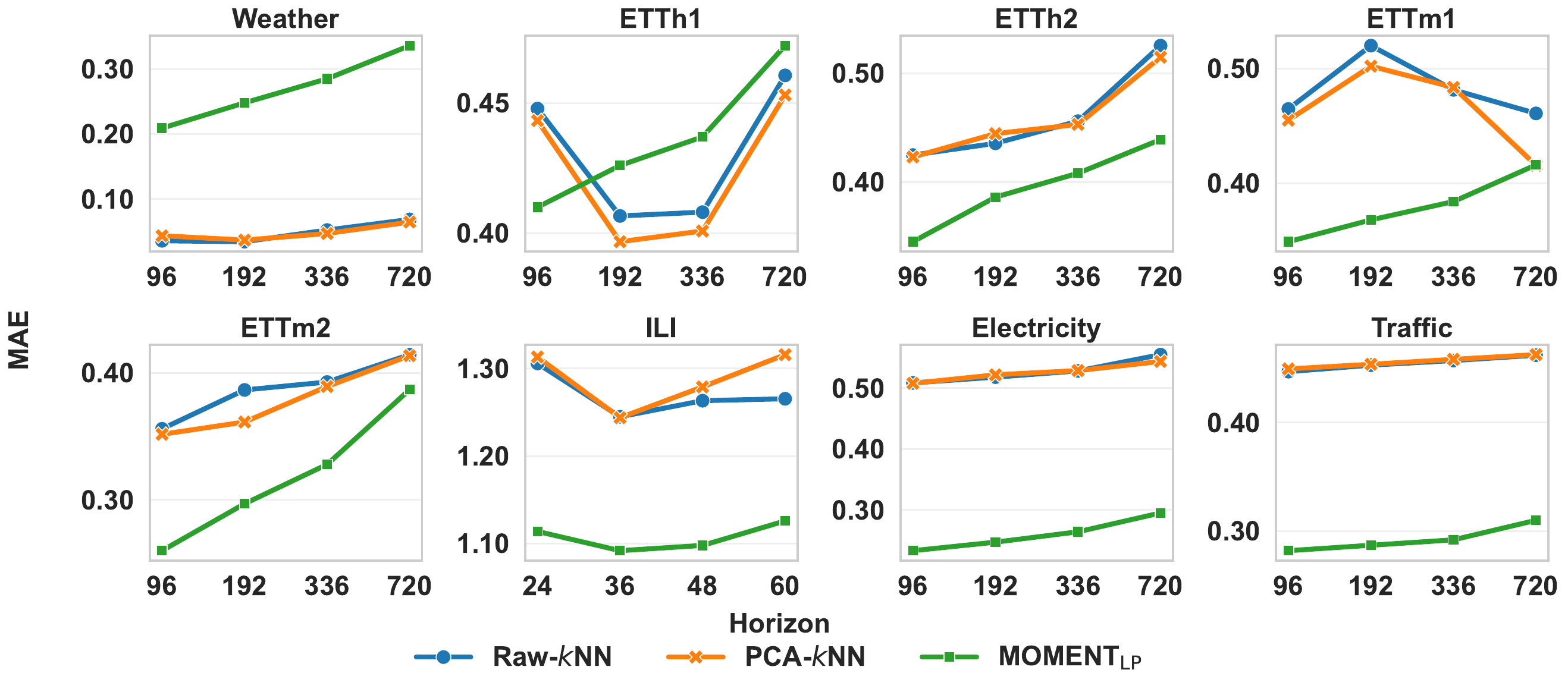}
    \caption{Comparison of long-horizon forecasting using test MAE across datasets and forecast horizons. It compares Raw-$k$NN and PCA-$k$NN with MOMENT across benchmark settings.}
    \label{fig:long_horizon_mae}
\end{figure}


\begin{table}[htbp]
  \caption{Long-horizon forecasting results. We report test MSE and MAE for Raw-$k$NN, PCA-$k$NN, and MOMENT; MOMENT results are taken from Table 2 of \citep{goswami2024moment}.}
  \label{tab:long_horizon_results}
  \centering
  \footnotesize
  \begin{tabular}{@{}lccccccc@{}}
    \toprule
    \multicolumn{2}{c}{Setting} &
    \multicolumn{2}{c}{Raw-$k$NN} &
    \multicolumn{2}{c}{PCA-$k$NN} &
    \multicolumn{2}{c}{MOMENT$_{\mathrm{LP}}$} \\
    \cmidrule(r){1-2}
    \cmidrule(r){3-4}
    \cmidrule(r){5-6}
    \cmidrule(l){7-8}
    Dataset & Horizon & MSE & MAE & MSE & MAE & MSE & MAE \\
    \midrule
    \multirow{4}{*}{Weather}
      & 96  & 0.023 & 0.036 & 0.121 & 0.043 & 0.154 & 0.209 \\
      & 192 & 0.002 & 0.034 & 0.003 & 0.037 & 0.197 & 0.248 \\
      & 336 & 0.055 & 0.052 & 0.160 & 0.047 & 0.246 & 0.285 \\
      & 720 & 0.095 & 0.068 & 0.103 & 0.064 & 0.315 & 0.336 \\
    \midrule
    \multirow{4}{*}{ETTh1}
      & 96  & 0.287 & 0.448 & 0.281 & 0.443 & 0.387 & 0.410 \\
      & 192 & 0.237 & 0.407 & 0.227 & 0.397 & 0.410 & 0.426 \\
      & 336 & 0.239 & 0.408 & 0.229 & 0.401 & 0.422 & 0.437 \\
      & 720 & 0.305 & 0.461 & 0.294 & 0.453 & 0.454 & 0.472 \\
    \midrule
    \multirow{4}{*}{ETTh2}
      & 96  & 0.289 & 0.425 & 0.286 & 0.423 & 0.288 & 0.345 \\
      & 192 & 0.306 & 0.436 & 0.309 & 0.444 & 0.349 & 0.386 \\
      & 336 & 0.332 & 0.456 & 0.318 & 0.453 & 0.369 & 0.408 \\
      & 720 & 0.433 & 0.525 & 0.410 & 0.515 & 0.403 & 0.439 \\
    \midrule
    \multirow{4}{*}{ETTm1}
      & 96  & 0.314 & 0.465 & 0.297 & 0.455 & 0.293 & 0.349 \\
      & 192 & 0.382 & 0.520 & 0.356 & 0.502 & 0.326 & 0.368 \\
      & 336 & 0.325 & 0.481 & 0.328 & 0.484 & 0.352 & 0.384 \\
      & 720 & 0.304 & 0.461 & 0.247 & 0.416 & 0.405 & 0.416 \\
    \midrule
    \multirow{4}{*}{ETTm2}
      & 96  & 0.212 & 0.356 & 0.206 & 0.352 & 0.170 & 0.260 \\
      & 192 & 0.245 & 0.387 & 0.215 & 0.361 & 0.227 & 0.297 \\
      & 336 & 0.252 & 0.393 & 0.246 & 0.389 & 0.275 & 0.328 \\
      & 720 & 0.279 & 0.415 & 0.275 & 0.414 & 0.363 & 0.387 \\
    \midrule
    \multirow{4}{*}{ILI}
      & 24 & 3.038 & 1.306 & 3.065 & 1.313 & 2.728 & 1.114 \\
      & 36 & 2.723 & 1.245 & 2.728 & 1.244 & 2.669 & 1.092 \\
      & 48 & 2.915 & 1.263 & 2.848 & 1.279 & 2.728 & 1.098 \\
      & 60 & 2.872 & 1.265 & 2.931 & 1.316 & 2.883 & 1.126 \\
    \midrule
    \multirow{4}{*}{ECL}
      & 96  & 0.476 & 0.508 & 0.470 & 0.508 & 0.136 & 0.233 \\
      & 192 & 0.493 & 0.518 & 0.493 & 0.522 & 0.152 & 0.247 \\
      & 336 & 0.509 & 0.528 & 0.502 & 0.529 & 0.167 & 0.264 \\
      & 720 & 0.545 & 0.555 & 0.520 & 0.544 & 0.205 & 0.295 \\
    \midrule
    \multirow{4}{*}{Traffic}
      & 96  & 0.396 & 0.447 & 0.392 & 0.449 & 0.391 & 0.282 \\
      & 192 & 0.405 & 0.453 & 0.400 & 0.454 & 0.404 & 0.287 \\
      & 336 & 0.408 & 0.457 & 0.405 & 0.458 & 0.414 & 0.292 \\
      & 720 & 0.409 & 0.462 & 0.410 & 0.463 & 0.450 & 0.310 \\
    \bottomrule
  \end{tabular}
\end{table}

\begin{table}[htbp]
  \caption{Classification accuracy across 91 UCR datasets. We report summary statistics for Raw-$k$NN, PCA-$k$NN, and MOMENT; MOMENT results are taken from Table 4 of \citep{goswami2024moment}.}
  \label{tab:classification_results}
  \centering
  \small
  \begin{tabular}{@{}lccc@{}}
    \toprule
    Statistic & Raw-$k$NN & PCA-$k$NN & MOMENT$_0$ \\
    \midrule
    Mean   & 0.748 & 0.748 & 0.794 \\
    Median & 0.753 & 0.760 & 0.815 \\
    Std.   & 0.153 & 0.151 & 0.147 \\
    \bottomrule
  \end{tabular}
\end{table}


\begin{table}[htbp]
  \caption{Short-horizon forecasting results measured using sMAPE. We compare Raw-$k$NN and PCA-$k$NN transfer results with MOMENT results from Table 3 of \cite{goswami2024moment} under the same source-to-target evaluation setting.}
  \label{tab:short_horizon_results_transfer}
  \centering
  \small
  \begin{tabular}{@{}llcccc@{}}
    \toprule
    \multicolumn{2}{c}{Setting} &
    \multirow{2}{*}{Raw-$k$NN} &
    \multirow{2}{*}{PCA-$k$NN} &
    \multicolumn{2}{c}{MOMENT$_{\mathrm{LP}}$} \\
    \cmidrule(lr){1-2}
    \cmidrule(lr){5-6}
    Dataset & Frequency & & & M4 & FR \\
    \midrule
    \multirow{3}{*}{M3}
      & Yearly    & 22.14 & 22.50 & 16.74 & 16.97 \\
      & Quarterly & 17.67 & 17.24 & 10.09 & 10.62 \\
      & Monthly   & 21.05 & 21.77 & 16.04 & 16.90 \\
    \midrule
    \multirow{3}{*}{M4}
      & Yearly    & 36.48 & 34.91 & --    & 14.84 \\
      & Quarterly & 34.48 & 34.93 & --    & 12.02 \\
      & Monthly   & 31.44 & 32.86 & --    & 15.80 \\
    \bottomrule
  \end{tabular}
\end{table}


\begin{table}[htbp]
  \caption{Short-horizon forecasting results measured using sMAPE under the within-group evaluation setting. We compare Raw-$k$NN and PCA-$k$NN with MOMENT results from Table 3 of \cite{goswami2024moment}.}
  \label{tab:short_horizon_results_within_group }
  \centering
  \small
  \begin{tabular}{@{}llcccc@{}}
    \toprule
    \multicolumn{2}{c}{Target series} &
    \multirow{2}{*}{Raw-$k$NN} &
    \multirow{2}{*}{PCA-$k$NN} &
    \multicolumn{2}{c}{MOMENT$_{\mathrm{LP}}$} \\
    \cmidrule(lr){1-2}
    \cmidrule(lr){5-6}
    Dataset & Frequency & & & M4 & FR \\
    \midrule
    \multirow{3}{*}{M3}
      & Yearly    & 20.56 & 21.32 & 16.74 & 16.97 \\
      & Quarterly & 14.33 & 14.38 & 10.09 & 10.62 \\
      & Monthly   & 15.78 & 16.17 & 16.04 & 16.90 \\
    \midrule
    \multirow{3}{*}{M4}
      & Yearly    & 16.65 & 16.69 & --    & 14.84 \\
      & Quarterly & 14.51 & 14.84 & --    & 12.02 \\
      & Monthly   & 15.37 & 15.53 & --    & 15.80 \\
    \bottomrule
  \end{tabular}
\end{table}



\subsection{Real-data radius-scaling diagnostics}
\label{app:real-data-diagnostics}

We complement the predictive evaluations in Section~\ref{exp3: Real-data_evaluation} with radius-scaling and local-mass diagnostics on representative real-data benchmarks. We use standardized input features for Raw-$k$NN and lower-dimensional PCA representations for PCA-$k$NN.

\paragraph{Diagnostic settings.}
For each dataset, we use at most 2000 reference samples and 300 query points. When these limits are exceeded, we sample rows uniformly without replacement and retain their original order. The short-horizon diagnostics use the source-to-target transfer setting, with M4 groups providing reference samples and the corresponding M3 groups providing query points. Standardization and PCA are fitted on the selected reference samples and applied to the query points. PCA uses up to 32 components, subject to the available sample size and feature dimension. We vary $n$ using reference-sample prefixes corresponding to fractions $\{0.25,0.4,0.6,0.8,1.0\}$. For each eligible prefix, we combine neighbor counts $\{3,5,9,15,21,31\}$ with rounded values of $n^\beta$ for $\beta\in\{0.3,0.5,0.7\}$, retaining only counts smaller than $n$. For the local-mass diagnostic, we select radii at 12 equally spaced quantile levels between the 5th and 35th percentiles of the positive query-to-reference distances. We use 200 bootstrap replicates with a block size of 16 selected query rows, reduced when fewer query points are available. The random seed is 13.

\paragraph{Radius-scaling diagnostic.}
Using the notation from Section~\ref{sec:setup}, we denote the $k$-NN radius around a query point $x$ by $R_{n,k}(x) := \lVert \mathcal X_{(k)}^n(x)-x \rVert$, where $n$ is the sample size and $k$ is the number of nearest neighbors. We compute this radius in either the standardized observation space or the PCA representation. We vary $n$ and $k$, compute the average radius over query points, and fit its slope versus $k/n$ on a log-log scale. This diagnostic examines how the empirical radius varies with $k/n$, motivated by the canonical scale $(k/n)^{1/s}$ under the theoretical assumptions, where $s$ is the local geometric dimension.

\paragraph{Local-mass diagnostic.}
We also examine how empirical mass accumulates around the query points. For each radius $r$, we compute $\widehat{\mu}(B(x,r))$, the empirical fraction of reference sample points falling in the ball $B(x,r)$. We average this quantity over query points and fit the logarithm of the average empirical mass versus $\log r$. The resulting slope summarizes mass growth over the evaluated radius range.

\paragraph{Bootstrap intervals.}
We obtain 95\% block-bootstrap intervals for the radius slopes by resampling blocks of consecutive rows in the selected query set. Within each replicate, we use the same resampled query indices across all $(n,k)$ configurations and refit the slope. The reference samples and fitted representations remain fixed. These intervals summarize variability under query-block resampling; they do not include variability from resampling the reference data or refitting the representations. Consecutive selected query rows need not correspond to consecutive observations in the original time series.

\paragraph{Results.}
Tables~\ref{tab:radius_long}, \ref{tab:radius_short}, and \ref{tab:radius_classification} report the fitted radius slopes, their 95\% block-bootstrap intervals, and the local-mass slopes for long-horizon forecasting, short-horizon forecasting, and UCR classification, respectively. Across these benchmarks, the fitted radius slopes are positive for both Raw-$k$NN and PCA-$k$NN representations, with the reported intervals bounded away from zero. This indicates that the average empirical $k$-NN radius shrinks as $k/n$ becomes smaller over the evaluated range, qualitatively consistent with the canonical scaling $(k/n)^{1/s}$. The PCA representation gives larger fitted radius slopes on these datasets. The local-mass fits also yield finite, positive empirical local exponents over the evaluated radius range. These diagnostics connect the real-data results to the radius-scaling and local-mass behavior studied in the theory, without establishing the exact exponent $1/s$ or verifying the full local-mass and $\alpha$-mixing assumptions. 


\begin{table}[t]
\centering
\small
\setlength{\tabcolsep}{5pt}
\caption{Radius-scaling and local-mass diagnostics for long-horizon forecasting. The 95\% block-bootstrap intervals apply to the radius slopes.}
\label{tab:radius_long}
\begin{tabular}{@{}ccccc@{}}
\toprule
Dataset & Representation & Radius slope & 95\% CI & Local-mass slope \\
\midrule
\multirow{2}{*}{Weather} & Raw-$k$NN & 0.176 & [0.165, 0.185] & 2.153 \\
 & PCA-$k$NN & 0.237 & [0.221, 0.254] & 1.989 \\ \midrule
\multirow{2}{*}{ETTh1} & Raw-$k$NN & 0.213 & [0.204, 0.224] & 2.601 \\
 & PCA-$k$NN & 0.227 & [0.214, 0.239] & 2.441 \\ \midrule
\multirow{2}{*}{ETTh2} & Raw-$k$NN & 0.158 & [0.133, 0.186] & 3.662 \\
 & PCA-$k$NN & 0.183 & [0.152, 0.217] & 3.337 \\ \midrule
\multirow{2}{*}{ETTm1} & Raw-$k$NN & 0.238 & [0.223, 0.251] & 2.139 \\
 & PCA-$k$NN & 0.247 & [0.230, 0.261] & 2.093 \\ \midrule
\multirow{2}{*}{ETTm2} & Raw-$k$NN & 0.199 & [0.178, 0.221] & 3.026 \\
 & PCA-$k$NN & 0.201 & [0.180, 0.223] & 3.011 \\ \midrule
\multirow{2}{*}{Electricity} & Raw-$k$NN & 0.087 & [0.079, 0.095] & 5.451 \\
 & PCA-$k$NN & 0.108 & [0.097, 0.119] & 5.124 \\
\bottomrule
\end{tabular}
\end{table}

\begin{table}[t]
\centering
\small
\setlength{\tabcolsep}{5pt}
\caption{Radius-scaling and local-mass diagnostics for short-horizon forecasting. The 95\% block-bootstrap intervals apply to the radius slopes.}
\label{tab:radius_short}
\begin{tabular}{@{}ccccc@{}}
\toprule
Dataset & Representation & Radius slope & 95\% CI & Local-mass slope \\
\midrule
\multirow{2}{*}{M3 Yearly} & Raw-$k$NN & 0.184 & [0.167, 0.211] & 2.462 \\
 & PCA-$k$NN & 0.213 & [0.197, 0.239] & 2.394 \\ \midrule
\multirow{2}{*}{M3 Quarterly} & Raw-$k$NN & 0.172 & [0.162, 0.185] & 2.815 \\
 & PCA-$k$NN & 0.240 & [0.224, 0.256] & 2.617 \\ \midrule
\multirow{2}{*}{M3 Monthly} & Raw-$k$NN & 0.145 & [0.134, 0.160] & 3.128 \\
 & PCA-$k$NN & 0.206 & [0.193, 0.225] & 2.814 \\
\bottomrule
\end{tabular}
\end{table}

\begin{table}[t]
\centering
\small
\setlength{\tabcolsep}{5pt}
\caption{Radius-scaling and local-mass diagnostics for UCR classification. The 95\% block-bootstrap intervals apply to the radius slopes.}
\label{tab:radius_classification}
\begin{tabular}{@{}ccccc@{}}
\toprule
Dataset & Representation & Radius slope & 95\% CI & Local-mass slope \\
\midrule
\multirow{2}{*}{Crop} & Raw-$k$NN & 0.212 & [0.198, 0.232] & 2.111 \\
 & PCA-$k$NN & 0.214 & [0.200, 0.235] & 2.101 \\ \midrule
\multirow{2}{*}{ElectricDevices} & Raw-$k$NN & 0.080 & [0.061, 0.095] & 8.944 \\
 & PCA-$k$NN & 0.142 & [0.127, 0.155] & 4.619 \\ \midrule
\multirow{2}{*}{ECG5000} & Raw-$k$NN & 0.194 & [0.168, 0.219] & 2.619 \\
 & PCA-$k$NN & 0.211 & [0.187, 0.237] & 2.442 \\
\bottomrule
\end{tabular}
\end{table}

\subsection{Dataset availability}\label{Appendix: Dataset_availability}

We use two types of datasets in this paper: synthetic datasets and real-world time-series benchmark datasets. The synthetic datasets in Experiments~\ref{exp1:Moment convergence rate} and~\ref{exp2: Entropy estimation} are generated from the data-generating processes described in Appendices~\ref{app: Experiment 1 implementation details} and~\ref{app: Experiment 2 implementation details}, respectively, using the global random seed 20260504. The real-world datasets in Experiment~\ref{exp3: Real-data_evaluation} are taken from the Time Series PILE~\citep{goswami2024moment}, which includes datasets from the Informer long-horizon forecasting benchmark, the Monash time series forecasting archive, and the UCR classification archive. These benchmark datasets are publicly available from their original sources.

\subsection{Code availability}\label{Appendix: Code_availability}

Code for synthetic data generation, preprocessing, hyperparameter tuning, evaluation, and figure reproduction will be made publicly available upon publication. Experiments~\ref{exp1:Moment convergence rate} and~\ref{exp2: Entropy estimation} were run entirely on CPU on a cluster node with two AMD EPYC 9754 128-Core Processors, comprising 256 physical cores and 256 hardware threads. Reproducing these two experiments takes approximately 5 hours. Experiment~\ref{exp3: Real-data_evaluation} was run entirely on CPU on a Linux workstation with an AMD EPYC 9554 64-Core Processor, comprising 64 physical cores and 128 hardware threads, and takes approximately 4 hours to reproduce. No GPU acceleration was used.

\newpage
\section*{NeurIPS Paper Checklist}

\begin{enumerate}

\item {\bf Claims}
    \item[] Question: Do the main claims made in the abstract and introduction accurately reflect the paper's contributions and scope?
    \item[] Answer: \answerYes{} 
    \item[] Justification: See \textbf{Abstract} and \textbf{Our contributions} in Section~\ref{sec:intro}.

\item {\bf Limitations}
    \item[] Question: Does the paper discuss the limitations of the work performed by the authors?
    \item[] Answer: \answerYes{} 
    \item[] Justification: See Section~\ref{sec:discussions}.

\item {\bf Theory assumptions and proofs}
    \item[] Question: For each theoretical result, does the paper provide the full set of assumptions and a complete (and correct) proof?
    \item[] Answer: \answerYes{} 
    \item[] Justification: See Section~\ref{sec:assumptions} and Appendix~\ref{sec:proof}.

    \item {\bf Experimental result reproducibility}
    \item[] Question: Does the paper fully disclose all the information needed to reproduce the main experimental results of the paper to the extent that it affects the main claims and/or conclusions of the paper (regardless of whether the code and data are provided or not)?
        \item[] Answer: \answerYes{}
        \item[] Justification: The paper reports three experiments: two using synthetic data and one using publicly available real-world time-series benchmark datasets. It also provides the preprocessing, evaluation, and implementation details needed to reproduce the main results; see Experiments~\ref{exp1:Moment convergence rate},~\ref{exp2: Entropy estimation}, and~\ref{exp3: Real-data_evaluation}, and Appendix~\ref{Appendix: Dataset_availability} and~\ref{Appendix: Code_availability}.
        


\item {\bf Open access to data and code}
    \item[] Question: Does the paper provide open access to the data and code, with sufficient instructions to faithfully reproduce the main experimental results, as described in supplemental material?
    \item[] Answer: \answerYes{}
    \item[] Justification: See Appendix~\ref{Appendix: Dataset_availability} and ~\ref{Appendix: Code_availability}.

\item {\bf Experimental setting/details}
    \item[] Question: Does the paper specify all the training and test details (e.g., data splits, hyperparameters, how they were chosen, type of optimizer) necessary to understand the results?
    \item[] Answer: \answerYes{} 
    \item[] Justification: See Experiments~\ref{exp1:Moment convergence rate},~\ref{exp2: Entropy estimation}, and~\ref{exp3: Real-data_evaluation}, and Appendix~\ref{Appendix: Dataset_availability} and~\ref{Appendix: Code_availability}.
    
    

\item {\bf Experiment statistical significance}
    \item[] Question: Does the paper report error bars suitably and correctly defined or other appropriate information about the statistical significance of the experiments?
    \item[] Answer: \answerNo{} 
    \item[] Justification: Our experiment does not include a statistical significance check.

\item {\bf Experiments compute resources}
    \item[] Question: For each experiment, does the paper provide sufficient information on the computer resources (type of compute workers, memory, time of execution) needed to reproduce the experiments?
    \item[] Answer: \answerYes{} 
    \item[] Justification: See Appendix~\ref{Appendix: Code_availability}.
    
\item {\bf Code of ethics}
    \item[] Question: Does the research conducted in the paper conform, in every respect, with the NeurIPS Code of Ethics \url{https://neurips.cc/public/EthicsGuidelines}?
    \item[] Answer: \answerYes{} 
    \item[] Justification: The research conducted in the paper conform, in every respect, with the NeurIPS Code of Ethics.

\item {\bf Broader impacts}
    \item[] Question: Does the paper discuss both potential positive societal impacts and negative societal impacts of the work performed?
    \item[] Answer: \answerNA{} 
    \item[] Justification: There is no societal impact of the work performed.
    
\item {\bf Safeguards}
    \item[] Question: Does the paper describe safeguards that have been put in place for responsible release of data or models that have a high risk for misuse (e.g., pre-trained language models, image generators, or scraped datasets)?
    \item[] Answer: \answerNA{} 
    \item[] Justification: The paper poses no such risks.

\item {\bf Licenses for existing assets}
    \item[] Question: Are the creators or original owners of assets (e.g., code, data, models), used in the paper, properly credited and are the license and terms of use explicitly mentioned and properly respected?
    \item[] Answer: \answerYes{} 
    \item[] Justification: See Appendix~\ref{Appendix: Code_availability}.

\item {\bf New assets}
    \item[] Question: Are new assets introduced in the paper well documented and is the documentation provided alongside the assets?
    \item[] Answer: \answerNA{} 
    \item[] Justification: The paper does not release new assets.

\item {\bf Crowdsourcing and research with human subjects}
    \item[] Question: For crowdsourcing experiments and research with human subjects, does the paper include the full text of instructions given to participants and screenshots, if applicable, as well as details about compensation (if any)? 
    \item[] Answer: \answerNA{} 
    \item[] Justification: The paper does not involve crowdsourcing nor research with human subjects.

\item {\bf Institutional review board (IRB) approvals or equivalent for research with human subjects}
    \item[] Question: Does the paper describe potential risks incurred by study participants, whether such risks were disclosed to the subjects, and whether Institutional Review Board (IRB) approvals (or an equivalent approval/review based on the requirements of your country or institution) were obtained?
    \item[] Answer: \answerNA{} 
    \item[] Justification: The paper does not involve crowdsourcing nor research with human subjects.

\item {\bf Declaration of LLM usage}
    \item[] Question: Does the paper describe the usage of LLMs if it is an important, original, or non-standard component of the core methods in this research? Note that if the LLM is used only for writing, editing, or formatting purposes and does \emph{not} impact the core methodology, scientific rigor, or originality of the research, declaration is not required.
    \item[] Answer: \answerNA{} 
    \item[] Justification: The core method development in this research does not involve LLMs as any important, original, or non-standard components.

\end{enumerate}

\end{document}